\documentclass[11pt,letterpaper]{article}
\def\draft{1}
\def\doubleblind{1} 

\usepackage{arxiv}

\makeatletter
\@ifundefined{headeright}{}{}
\@ifundefined{undertitle}{}{}
\@ifundefined{shorttitle}{}{}
\makeatother

\usepackage{amsmath,amssymb,amsfonts,amsthm,mathtools}
\usepackage{thmtools}
\usepackage{bbm}
\usepackage{xcolor}
\usepackage{graphicx}
\usepackage{enumitem}
\usepackage{booktabs}
\usepackage{tabularx}
\usepackage{array}
\usepackage{algorithm}
\usepackage{algpseudocode}
\usepackage{enumitem}

\usepackage[utf8]{inputenc}
\usepackage[T1]{fontenc}
\usepackage{microtype}
\usepackage{amsthm}
\usepackage{amsthm, amsmath, amssymb, mathtools}
\usepackage{bm,bbm}
\usepackage{graphicx}
\usepackage{color,xcolor}
\usepackage{enumitem}
\usepackage{mathrsfs}
\usepackage[normalem]{ulem}
\usepackage{tabularx}
\usepackage{tikz}
\usetikzlibrary{shapes.geometric, arrows.meta, calc, backgrounds, positioning}
\usetikzlibrary{decorations.pathreplacing}
\usetikzlibrary{patterns}
\usepackage{pgfplots}
\pgfplotsset{compat=1.18}
\usepackage{amsmath, tikz-cd}
\usepackage{caption,comment}
\usepackage{booktabs}
\usepackage{algorithmicx}
\usepackage{algorithm}
\usepackage{algpseudocode}
\usepackage{fontawesome}
\usepackage[dvipsnames]{xcolor}
\usepackage[textsize=tiny]{todonotes}
\usepackage{standalone}
\usepackage{cancel}
\usepackage{xr-hyper}
\usepackage{hyperref}
\hypersetup{colorlinks=true,citecolor=Aquamarine, linkcolor=magenta, urlcolor=blue}
\usepackage{parskip}

\newcommand{\blind}[2]{{\ifnum\draft=1\color{purple}\fi \ifnum\doubleblind=1#2\fi\ifnum\doubleblind=0#1\fi\ifnum\doubleblind=2$\{$ #1 $\vert$ #2 $\}$\fi}}

\makeatletter
\newcommand{\algmargin}{\the\ALG@thistlm}
\algnewcommand{\parState}[1]{\State%
  \parbox[t]{\dimexpr\linewidth-\algmargin}{\strut #1\strut}}

\makeatletter
\@ifpackageloaded{hyperref}{%
  \hypersetup{colorlinks=true, allcolors=blue}%
}{%
  \usepackage[colorlinks=true, allcolors=blue]{hyperref}%
}
\@ifpackageloaded{cleveref}{}{%
  \usepackage[capitalize,noabbrev,nameinlink]{cleveref}%
}
\makeatother

\usepackage[numbers]{natbib}

\usepackage{amsthm}
\usepackage{thmtools,thm-restate}

\numberwithin{equation}{section}
\declaretheoremstyle[bodyfont=\it,qed=\qedsymbol]{noproofstyle}

\declaretheorem[name=Observation,numbered=no]{observation*}

\declaretheorem[numberlike=equation]{theorem}

\declaretheorem[name=Theorem,numbered=no]{theorem*}

\declaretheorem[numberlike=equation]{lemma}
\declaretheorem[name=Lemma,numbered=no]{lemma*}

\declaretheorem[numberlike=equation]{corollary}
\declaretheorem[name=Corollary,numbered=no]{corollary*}

\declaretheorem[name=Proposition,numbered=no]{proposition*}

\declaretheorem[numberlike=equation]{claim}
\declaretheorem[name=Claim,numbered=no]{claim*}

\declaretheorem[name=Conjecture,numbered=no]{conjecture*}

\declaretheorem[name=Question,numbered=no]{question*}

\declaretheoremstyle[bodyfont=\it]{defstyle} 

\declaretheorem[numberlike=equation,style=defstyle]{definition}
\declaretheorem[unnumbered,name=Definition,style=defstyle]{definition*}

\declaretheorem[unnumbered,name=Example,style=defstyle]{example*}

\declaretheorem[unnumbered,name=Notation=defstyle]{notation*}

\declaretheorem[unnumbered,name=Construction,style=defstyle]{construction*}

\declaretheoremstyle[]{rmkstyle} 

\newtheorem*{remark}{Remark}

\crefname{claim}{Claim}{Claims}
\crefname{fact}{Fact}{Facts}

\DeclarePairedDelimiterX{\divx}[2]{(}{)}{#1 \;\delimsize\|\; #2}

\usepackage{makecell}
\usepackage{epstopdf}
\usepackage{subcaption}
\usepackage{multirow}
\usepackage{threeparttable}
\usepackage{wrapfig}
\usepackage{pifont}
\usepackage{nicefrac}

\providecommand{\REQUIRE}{\Require}

\providecommand{\STATE}{\State}
\providecommand{\IF}[1]{\If{#1}}

\providecommand{\ELSE}{\Else}
\providecommand{\ENDIF}{\EndIf}
\providecommand{\FOR}[1]{\For{#1}}

\providecommand{\ENDFOR}{\EndFor}

\allowdisplaybreaks
\makeatletter
\@ifundefined{assumption}{\declaretheorem[numberlike=equation,style=defstyle]{assumption}}{}
\@ifundefined{hypothesis}{}{}
\makeatother

\crefname{assumption}{Assumption}{Assumptions}
\Crefname{assumption}{Assumption}{Assumptions}
\crefname{hypothesis}{Hypothesis}{Hypotheses}
\Crefname{hypothesis}{Hypothesis}{Hypotheses}

\DeclareMathOperator*{\argmin}{arg\,min}

\mathcode`*="8000
{\catcode`*=\active
\gdef*{\star}}

\newcommand{\COCO}{\textsf{COCO }}
\newcommand{\CExperts}{\textsf{Constrained Experts }}

\makeatletter
\@ifundefined{c@commentcounter}{\newcounter{commentcounter}}{}
\makeatother

\providecommand{\cmt}[1]{%
  \refstepcounter{commentcounter}%
  \textcolor{red}{\textbf{[(\thecommentcounter) AS: #1]}}%
}

\let\geq\geqslant
\let\leq\leqslant
\let\ge\geqslant
\let\le\leqslant

\newcommand{\lnorm}{\left\lVert}
\newcommand{\rnorm}{\right\rVert}

\providecommand{\Regret}{\mathsf{Regret}}
\providecommand{\Reg}{\mathsf{Regret}}
\providecommand{\CCV}{\mathsf{CCV}}

\providecommand{\diam}{\operatorname{diam}}
\providecommand{\proj}{\Pi}
\providecommand{\RR}{\mathbb R}

\providecommand{\blind}[2]{%
  {\ifnum\draft=1\color{purple}\fi
  \ifnum\doubleblind=1 #2\fi
  \ifnum\doubleblind=0 #1\fi
  \ifnum\doubleblind=2 $\{$ #1 $\vert$ #2 $\}$\fi}%
}

\title{A Geometric Approach to Constrained Online Learning}

\author{
  Dhruv Sarkar\textsuperscript{1,2} \quad
  Abhishek Sinha\textsuperscript{3} \\[3pt]
  \normalfont\small\textsuperscript{1}Indian Institute of Technology Kharagpur \\[-1pt]
  \normalfont\small\textsuperscript{2}Mohamed bin Zayed University of Artificial Intelligence \\[-1pt]
  \normalfont\small\textsuperscript{3}Tata Institute of Fundamental Research, Mumbai \\[2pt]
  \normalfont\small\texttt{dhruv.sarkar223@gmail.com} \quad \texttt{abhishek.sinha@tifr.res.in} \\[2pt]
  }

\date{}

\begin{document}

\maketitle

\begin{abstract}
We consider an online learning problem with time-varying adversarial constraints. At each round, a learner selects an action from a bounded convex decision set before observing a loss function and a constraint function, both of which are chosen by an adaptive adversary. The goal of the learner is to obtain the minimax-optimal regret for the loss functions relative to the best fixed action satisfying all constraints in hindsight
and to simultaneously minimize the cumulative constraint violation ($\CCV$) corresponding to the constraint functions.
Prior to this work, the best-known algorithms achieve $O(\log T)$ $\Regret$ and $O(\sqrt{T\log T})$ $\CCV$ for strongly convex losses 
and $O(\sqrt{T})$ $\Regret$ and $O(\sqrt{T}\log T)$ $\CCV$ for general convex losses. 
In this paper, we present an iterated projection-based algorithm \textsf{NP-OGD} that attains $O(\log T)$ $\Regret$ and $O(\log T)$ $\CCV$ for strongly convex losses - reducing the $\CCV$ from polynomial to logarithmic. For general convex losses, our algorithm achieves $O(\sqrt{T})$ $\Regret$ and $O(\sqrt{T})$ $\CCV,$ eliminating an extra logarithmic factor from prior bounds. The key to our analysis is a recent geometric result on the path length of self-contracted curves. In particular, we show that when appropriately lifted to a higher-dimensional space, the iterates produced by the nested projected gradient-descent algorithm satisfy a self-contraction property with respect to a non-standard norm. Furthermore, by utilizing a layered sphere-packing construction, we complement our achievability result by establishing an $\Omega\big(\frac{(\log T)^{\frac{d-1}{d+1}}}{\log \log T}\big)$ lower bound for $\CCV$ for strongly convex losses for any online no-regret algorithm. We further give stronger lower bounds for the \(\CCV\) in the convex-loss setting: \(\Omega(T^{\frac{d-1}{2(d+3)}})\) for weakly adaptive algorithms.  Finally, we study a special case of $\COCO$where the decision set is the probability simplex over $N$ experts. We show that a simple modification of the classical \textsf{Hedge} algorithm achieves the minimax-optimal regret bound of $O(\sqrt{T\log N})$ while incurring only $O(N)$ $\CCV$. Using a similar argument as $\mathsf{COCO}$, we also establish a matching minimax $\CCV$ lower bound for this case. 
\end{abstract}

\newpage
\tableofcontents
\newpage

\pagenumbering{arabic}

\section{Introduction}

Online convex optimization (OCO) is a standard framework for modelling and analyzing sequential decision
making under uncertainty \cite{hazan2022introduction}. In this framework, at the beginning of each round \(t \in [T]\), a learner \footnote{We will use the terms learner and online algorithm interchangeably.} selects an action
\(x_t\) from a closed, $d$-dimensional convex decision set $\mathcal{X} \subseteq \mathbb{R}^d$ of bounded Euclidean diameter $D$. Subsequently, the adversary reveals a $G$-Lipschitz convex loss function
\(f_t:\mathcal{X} \to\mathbb R\), and the learner incurs a cost of \(f_t(x_t)\) for round $t$.  
The learner's performance is measured using the standard $\Regret$ metric, which compares the cumulative cost incurred by the learner with that of the best fixed action in hindsight: \begin{eqnarray} \label{reg-def}
     \Reg_T
        :=
        \sum_{t=1}^Tf_t(x_t)
        -
        \min_{x\in \mathcal{X}}\sum_{t=1}^Tf_t(x).	
\end{eqnarray}
   
Many efficient algorithms are known for this problem achieving the minimax-optimal regret guarantee. In particular, for general convex and $\alpha$-strongly convex loss functions, it is well-known that the ubiquitous online gradient descent policy (\textsf{OGD}) achieves the minimax-optimal regret of $O(\sqrt{T})$ and $O(\log T)$ respectively with appropriate step sizes \cite{hazan2022introduction}. The online gradient descent policy belongs to the general class of online mirror descent (\textsf{OMD}) policies whose generic update at round $t+1$ takes the following form:
\begin{eqnarray} \label{gen-omd-0}
 x_{t+1}=\arg \min_{x \in {\mathcal{X}}} \big(\eta_t \langle \nabla f_t(x_t), x \rangle  + D(x, x_t)\big),	
\end{eqnarray}
where $\{\eta_t >0, t\geq 1\}$ is an appropriate step-size sequence and $D(\cdot, \cdot)$ is a proximal term, given by the Bregman divergence with respect to a strongly convex function \cite{hazan2022introduction}. For the \textsf{OGD} policy, the proximal term is simply equal to the squared Euclidean distance. Similarly, choosing $D(\cdot, \cdot)$ to be the KL-divergence yields the celebrated Exponential Weight (\textsf{Hedge}) policy.

Constrained Online Convex Optimization ($\mathsf{COCO}$) generalizes the OCO framework by incorporating long-term constraints \cite{sinha2024optimal, guo2022online}. Consistent with OCO, at the beginning of each round $t \in [T],$ the learner selects an action $x_t$ from a closed and convex decision set $\mathcal{X}.$ Subsequently, the adversary reveals two functions: a $G$-Lipschitz convex cost function $f_t: \mathcal{X} \mapsto \mathbb{R}$ and a closed convex, $G$-Lipschitz constraint function $g_t: \mathcal{X} \mapsto \mathbb{R}.$ 
 The constraint function $g_t$ corresponds to an instantaneous constraint of the form $g_t(x) \leq 0.$ Since the cost and constraint functions are revealed by the adaptive adversary after the learner selects its action at every round, the learner cannot possibly satisfy the constraint on every round, thus incurring a constraint violation of magnitude $(g_t(x_t))^+$ at round $t$. In the COCO problem, the informal goal of the learner is to minimize the $\Regret$ while simultaneously minimizing the cumulative constraint violation ($\CCV$). Specifically, we are interested in minimizing the $\CCV$ while achieving the minimax-optimal $\Regret$ bound. We focus on the dependence of the bound on the horizon length $T$ while ignoring their dependence on other parameters, \emph{e.g.}, the dimension $d,$ which is assumed to be a fixed constant throughout. 

To state the goal formally, define the sets
\begin{eqnarray} \label{feas-set}
        S_t
        :=
        \bigcap_{\tau=1}^t\{x \in \mathcal X:g_\tau(x)\le0\}, \quad t \geq 1,
\end{eqnarray}
denoting the intersection of all feasible sets (with respect to the revealed constraints) after $t$ rounds, with $S_0 = \mathcal{X}.$ Clearly, the sets $\{S_t\}_{t\geq 1}$ are closed, convex and nested:
\[
        \mathcal X=S_0\supseteq S_1\supseteq S_2\supseteq\cdots\supseteq S_T \equiv \mathcal{X}^* (\text{say}).
\]
This nested structure is one of the key
geometric features of the problem that we will exploit in the sequel. We make the following assumption.
\begin{assumption}[$\mathsf{Common~Feasibility}$] \label{common-feasibility}
	There exists a fixed action that satisfies all constraints, \emph{i.e.,} $S_T \neq \emptyset$.
\end{assumption}
 The performance of an online $\COCO$algorithm is measured against the set of all fixed offline actions in $S_T = \mathcal{X}^*$ that satisfy all constraints. Without Assumption \ref{common-feasibility}, the benchmark set is empty and regret with respect to feasible actions is undefined. Since ideally we want to minimize the cumulative cost while satisfying the adversarial constraints, our formal objective is to control the following two metrics simultaneously: 
\begin{eqnarray} \label{reg-ccv-def}
        \Reg_T
        :=
        \sum_{t=1}^Tf_t(x_t)
        -
        \min_{x\in \mathcal{X}^*}\sum_{t=1}^Tf_t(x), \quad \CCV_T := \sum_{t=1}^T (g_t(x_t))^+,
\end{eqnarray}
where we define $(z)^+ \equiv (z)_+ \equiv \max(0,z), ~z \in \mathbb{R}.$

%
%
%
\paragraph{Key algorithmic ideas:}
Our primary intuition is that had we known the common feasible set $\mathcal{X}^\star$ \emph{before} the game begins, we could have just run a standard OCO policy, \emph{e.g.,} Online Mirror Descent (\textsf{OMD}), with the action set restricted to the common feasible set $\mathcal{X}^*.$ This strategy trivially achieves the minimax-optimal $\Regret$ and a zero $\CCV$. However, since the constraints are revealed in a sequential fashion, the online policy gets to learn the final feasible set $S_T = \mathcal{X}^*$ only at the end of the horizon. This motivates the following online approximation where, at each round $t\ge 1,$ the feasible set $\mathcal{X}^*$ is approximated with $S_t$ - the set of all actions satisfying the constraints up to round $t$. We incorporate the knowledge of the current feasible set $S_t$ into the generic \textsf{OMD} policy and consider the following family of algorithms whose action at round $t$ is given by the following update rule:  
\begin{eqnarray} \label{gen-omd}
 x_{t+1}=\arg \min_{x \in {S_t}} \big(\eta_t \langle \nabla f_t(x_t), x \rangle  + D(x, x_t)\big),	
\end{eqnarray}
 The update rule \eqref{gen-omd} is motivated by the generic iteration \eqref{gen-omd-0} where instead of minimizing the objective over the entire decision set at round $t,$ we minimize it over the current feasible set $S_t$. Since the action $x_{t+1} \in S_t$ may not belong to the common feasible set $S_T,$ the online policy \eqref{gen-omd} may incur a constraint violation, which we need to bound, along with its $\Regret$. We apply the above generic algorithmic template to two concrete settings (1) general \textsf{COCO} and (2) a special case \textsf{Constrained Experts}.

 \subsection{Main Results:}

\paragraph{1. \textsf{COCO}:} Specializing the above policy for the \COCO problem upon choosing $D(\cdot, \cdot)$ to be the squared Euclidean distance, the update in \eqref{gen-omd} reduces to \textsf{Nested Projected Online Gradient Descent} algorithm (\textsf{NP-OGD}), formally described in Algorithm \ref{alg:nested-projected-ogd}.  Any algorithm in this class takes a gradient step on the last revealed cost function at each round and then projects it onto the intersection of the feasible sets revealed so far. In this case, the successive iteration \eqref{gen-omd} simplifies to: 
\begin{eqnarray} \label{pgd-generic}
	    x_{t+1}=\Pi_{S_t}(x_t - e_t), 
\end{eqnarray}
    where the vector $e_t$ is equal to a scaled (sub)gradient of the loss function $f_t$ and $\Pi_{C}(\cdot)$ is the standard Euclidean projection operator on the closed convex set $C$. This algorithm was considered earlier by \cite{vaze2026sqrt} where they showed that it achieves $O(\sqrt{T})$ regret and an \emph{instance}-dependent $\CCV$, which can be bounded by a constant for some special instances. However,  for arbitrary convex constraint sets, the $\CCV$ bound could still be $\Omega(T)$ in the worst case.   
        
We first show that, with an appropriate choice for the step size sequence, the \textsf{NP-OGD} algorithm achieves the minimax-optimal regret of $O(\sqrt{T})$ for convex and $O(\log T)$ regret for strongly convex losses respectively.  
%
%
%
%
Hence, the main remaining challenge is to control the $\CCV$ for this class of algorithms. From the Lipschitzness of the constraint functions, it follows that the $\CCV$ is bounded by the total movement cost of the algorithm, given by $\sum_{t=1}^{T} \lnorm x_{t+1}-x_t \rnorm_2.$  Due to successive projection operations, this
movement cost \emph{cannot} be bounded by the step lengths alone as the 
projection onto the smaller set \(S_t\) may create significant additional movement. This leads to the key technical contribution of the paper described below. 

\paragraph{Approximate Self-contraction of \textsf{NP-OGD} Trajectory:} We obtain a new movement bound for the perturbed nested projection trajectories \eqref{pgd-generic} by connecting them to self-contracted curves through a lifting argument (Lemma \ref{lem:robust-nested-projection}). As discussed above, the key technical difficulty for bounding the movement cost is that projection onto successively shrinking feasible regions introduces additional movement beyond the gradient step itself. The movement lemma (Lemma \ref{lem:robust-nested-projection}) shows that these additional projection effects do not accumulate arbitrarily; instead they inherit a finite-length property through an approximate \emph{self-contraction} property (see Section \ref{prelims} for the associated definitions). In particular, we show that the movement cost for the generic iterations \eqref{pgd-generic} with arbitrary nested convex sets can be upper-bounded as:
\begin{eqnarray} \label{movement-bound-intro}
        \sum_{t=1}^T\|x_{t+1}-x_t\|_2
        \le
        C_{d}
        \left(
            D+\sum_{t=1}^T\|e_t\|_2
        \right),
\end{eqnarray}
where \(C_d\) is a constant depending only on the dimension $d$ and independent of $T$. To the best of our knowledge, this is the \emph{first} movement bound for the ubiquitous gradient descent algorithm with arbitrary nested projections. This result immediately yields the following improvements over the state-of-the-art. 

 \begin{enumerate}
	\item For strongly convex loss functions, Theorem \ref{thm:main} improves the $\CCV$ bound to $O(\log T)$ from the previous best-known bound of $O(\sqrt{T})$ while retaining the minimax-optimal $O(\log T) $ regret guarantee.  
	\item For general convex losses, Theorem \ref{thm:main} also improves the best-known $\CCV$ bound by a factor of $(\log T)$ while retaining the optimal $O(\sqrt{T})$ regret guarantee. Removing the logarithmic factor from $\CCV$ has been open in the literature and there have been previous attempts to close the gap \cite{ferreira2025optimalboundsadversarialconstrained}. However, the proof was later found to contain an error. 
	\item Furthermore, by directly incorporating the constraint function in the objective of the iterated projected gradient algorithm, Section \ref{structure-section} yields $O(\log T)$ $\Regret$ and $O(\log T)$ $\CCV$ under an additional assumption of non-negative terminal $\CCV$, for general convex loss and strongly convex constraint functions.   
\end{enumerate}

The proposed algorithm is anytime, \emph{i.e.,} does not need to know the horizon length $T.$ Furthermore, for  items 1 and 2 above, the algorithm only requires quasi-convexity of the constraint functions \footnote{Recall that a function is called quasi-convex if its level sets are convex. All convex functions are quasi-convex \cite{boyd}.}. 
	The most computationally expensive operation at each step is the projection step at each round, which can be computed efficiently for many constraint functions (\emph{e.g.,} $\ell_2$-norms).
\paragraph{2. The $\CExperts$problem:} To supplement our main results for the $\COCO$problem, in Section \ref{experts-sec}, we consider a special case where the decision set $\mathcal{X}$ is restricted to be the probability simplex $\Delta_N$ and the loss and constraint functions are linear, \emph{i.e.,} $f_t(p) = \langle l_t, p \rangle, g_t(p) = \langle g_t,p\rangle, p \in \Delta_N$, for some $N$-length vectors $l_t, g_t \in [-1,1]^N$ \footnote{To emphasize the fact that here the actions are probability vectors, we denote them by the variable $p$ instead of $x$ as in \COCO.}. Each coordinate of the loss vector corresponds to an \emph{Expert} where the online algorithm chooses a convex combination of the experts at each round. The set of feasible experts up to round $t$ satisfying the constraints revealed so far is given by 
\[S_t = \{i \in [N]: g_\tau(i) \leq 0, 1\leq \tau \leq t\}, \quad t \geq 1. \]

In this case, choosing the proximal function \(D(\cdot,\cdot)\) in~\eqref{gen-omd} to be the
KL-divergence and restricting the update to the active simplex
\[
        \Delta(S_t):=\{p\in\Delta_N:p(i)=0\ \forall i\notin S_t\},
\]
the one-step update specializes to
\[
        p_{t+1}
        =
        \argmin_{p\in\Delta(S_t)}
        \left\{
            \eta_t\langle \ell_t,p\rangle
            +
            D_{\rm KL}(p\|p_t)
        \right\}.
\]
Solving the first-order optimality conditions gives the following $\textsf{Active Hedge}$ policy
\[
        p_{t+1}(i)
        =
        \frac{p_t(i)\exp(-\eta_t\ell_t(i))}
        {\sum_{j\in S_t}p_t(j)\exp(-\eta_t\ell_t(j))}
        \mathbf 1\{i\in S_t\}.
\]
Thus, using the KL-divergence for $D(\cdot, \cdot)$ in~\eqref{gen-omd} leads to the active Hedge algorithm, formally described later in Algorithm~\ref{alg:active-hedge-full-info}.
Theorem \ref{thm:active-hedge-distribution-version} shows that Algorithm \ref{alg:active-hedge-full-info} achieves the minimax-optimal $O(\sqrt{T \ln N})$ $\Regret$ and $O(N)$ $\CCV$, which is independent of the horizon length $T.$ This improves upon the prior result \cite[Theorem 2]{sinha2026beyond}, for sufficiently large $T$ (when $T \geq \tilde{\Omega}(N^2)$), which gave an algorithm achieving $O(\sqrt{T \ln N})$ $\Regret$ and $O(\sqrt{T} \ln N \ln T)$ $\CCV$. Compared to the intricate geometric argument for the $\COCO$problem, the proof technique for the constrained experts problem is elementary where  we exploit the essentially discrete nature of the problem. The regret bound follows upon modifying the standard proof arguments for the Hedge policy with a time-varying potential which keeps track of the currently feasible experts. The $\CCV$ bound follows from the observation that as soon as an expert violates a constraint, it is eliminated from further consideration. Surprisingly, as we discuss below, these bounds are tight for large $T$.   

\paragraph{3. Lower bounds:} In Sections~\ref{lower_bound_section} and~\ref{sec:cexperts-lb}, we establish
minimax lower bounds for the \(\COCO\) problem with strongly convex losses and
for the \(\CExperts\) problem, respectively.  These results complement the lower
bound of~\cite{sinha2024optimal}, who showed an
\((\Omega(\sqrt T),\Omega(\sqrt T))\) lower bound for the
\((\Regret,\CCV)\) pair in \(\COCO\) with general convex losses and general
convex constraints.  Their construction, however, is inherently
high-dimensional: it uses a fresh coordinate in each phase and therefore
requires dimension at least of order \(T\).  Consequently, it does not rule out
substantially smaller \(\CCV\) in the fixed-dimensional regime, which is the
regime studied in this paper.  Our lower bounds address precisely this gap by
showing that even when the dimension \(d\) is fixed, the geometry of nested
feasible regions imposes nontrivial lower bounds on the attainable \(\CCV\).

The common proof idea for the regret-budgeted lower bounds in
$\COCO$ (Section~\ref{lower_bound_section}) and $\CExperts$ (Section~\ref{sec:cexperts-lb}) is to consider a
sequence of phases with geometrically increasing lengths. During phase $j$, the cost function remains invariant throughout which is minimized at a chosen feasible point $z_j$. To maintain the no-regret property, we show that any online algorithm must play an action close to $z_j$ during phase $j$. Immediately after it happens, the adversary reveals an affine constraint function that makes action $z_j$ infeasible and thus the online algorithm incurs a constraint violation. The key challenge is to construct a set of \emph{well-separated} constraint functions so that we can continue the above construction for many phases, resulting in a tight $\CCV$ lower bound. Towards this, we utilize the geometry of $d$-dimensional Euclidean ball for $\COCO$ with a layered sphere-packing construction (see Figure \ref{fig:layered_cuts}). The construction for the $\CExperts$problem is considerably simpler where the constraints correspond to separate arms.  

We also prove an additional lower-bound results for $\COCO$ in the convex-loss setting.
In Section~\ref{cvx-loss-lb}, we show that any algorithm
satisfying a weak interval-regret guarantee must incur polynomial \(\CCV\) of $\Omega(T^{(d-1)/(2(d+3))})$;
this gives a convex-loss lower bound beyond the strongly convex construction of
Section~\ref{lower_bound_section}.  Thus, while our
upper bounds show that NP-OGD improves the best-known \(\CCV\) guarantees, the
lower bounds show that its general convex-loss guarantee cannot be improved to
a polylogarithmic bound by analysis alone. Independent and concurrent work by \cite{balasundaram2026lowerboundcumulativeconstrained} shows a sharper lower bound of $\Omega(T^{(d-1)/(2d})$ for the specific \textsf{NP-OGD} algorithm.




\begin{table*}[t]
\centering
\renewcommand{\arraystretch}{1.5} 
\small 

\begin{tabularx}{\textwidth}{lllll} 
\toprule
\textbf{Reference} &
\makecell{\textbf{$(\Reg,\CCV)$}\\\textbf{Convex Loss}} &
\makecell{\textbf{$(\Reg,\CCV)$}\\\textbf{Strongly Convex Loss}} &
\makecell{\textbf{Computation} \\ \textbf{at round $t$}} &
\textbf{Assumptions} \\
\midrule
\cite{guo2022online} & $\big(O(\sqrt{T}), O(T^{3/4})\big)$ & $\big(O(\log{T}), O(\sqrt{T \log T})\big)$ & Convex Opt. & ---\\
\cite{sinha2024optimal} & $\big(O(\sqrt{T}), O(\sqrt{T}\log T)\big)$ & $\big(O(\log{T}), O(\sqrt{T \log T})\big)$ & $\mathsf{Proj}_{\mathcal{X}}(\cdot)$ & ---\\
\cite{sinha2024optimal} & $\big(\Omega(\sqrt{T}), \Omega(\sqrt{T})\big)$ & --- & \textrm{Lower bound} & $d \ge T$\\
\cite{vaze2026sqrt} & $\big(O(\sqrt{T}), \textrm{instance~dependent}) $ & --- & $\mathsf{Proj}_{S_t}(\cdot)$ & ---\\
\textbf{This paper} & $\big(O(\sqrt{T}), O(\sqrt{T})\big)$ & $\big(O(\log{T}), O(\log T)\big)$ & $\mathsf{Proj}_{\mathcal{S}_t}(\cdot)$ & --- \\
\textbf{This paper} & $\big(O(\log{T}), O(\log T)\big)$ & --- & $\mathsf{Proj}_{\mathcal{S}_t}(\cdot)$ & \makecell{\small{$\CCV\geq 0$, Strongly} \\  \small{convex constraint}} \\
\textbf{This paper} & --- &$\CCV_T \geq \frac{\Omega(\log T)^{\frac{d-1}{d+1}}}{\ln\ln T}$ &  \textrm{Lower bound}  & \textrm{No-regret policies} \\
\textbf{This paper}
&
\(\CCV_T\ge \Omega(T^{\frac{d-1}{2(d+3)}})\)
&
---
&
Lower bound
&
Weak interval regret
\\
\bottomrule
\end{tabularx}
\caption{\small{Comparison of our results with prior work for $\mathsf{COCO}$. $\mathcal{S}_t$ denotes the feasible set up to round $t.$}}
\label{tab:comparison}
\end{table*}


\begin{table*}[t]
\centering
\renewcommand{\arraystretch}{1.5} 
\small 
\begin{tabularx}{0.94\linewidth}{
    l
    c
    c
    c
    >{\centering\arraybackslash}X
}
\toprule
\textbf{Reference}
&
\textbf{$\Regret$}
&
\textbf{$\CCV$}
&
\textbf{Algorithm}
&
\textbf{Assumptions}
\\
\midrule
\cite{sinha2026beyond}
&
\(O(\sqrt{T\log N})\)
&
\(O(\sqrt{T\log N}\log T)\)
&
Hedge
&
 \makecell{\small{Linear losses}\\/constraints}
\\
\textbf{This paper}
&
\(O(\sqrt{T\log N})\)
&
\(O(N)\)
&
Active Hedge on \(\Delta(S_t)\)
&
 \makecell{\small{Common feasible}\\ expert}
\\
\textbf{This paper}
&
--
&
\(\Omega(N)\)
&
Lower bound
&
No-regret policies
\\
\bottomrule
\end{tabularx}
\vspace{2mm}
\caption{\small{Comparison of our results with prior work for the $\CExperts$problem. $N$ denotes the number of experts.}}
\label{tab:constrained-experts-comparison}
\end{table*}

\paragraph{Organization.}
Section \ref{related-work} discusses relevant literature to put our results into proper context. Section \ref{prelims} briefly reviews a recent geometric result on self-contracted curves that will be key to our analysis. Section \ref{sec:nested-projected-ogd}
presents the general class of Nested Projected OGD algorithms and presents our main result on the $\Regret$ and
$\CCV$ guarantees for both general convex and strongly convex losses. Section \ref{lower_bound_section} presents a lower bound to $\CCV$ for strongly convex losses. Section~\ref{cvx-loss-lb} gives a weak-adaptivity lower bound for
convex losses. Section \ref{experts-sec} considers the $\CExperts$ problem and presents an algorithm, upper, and lower bounds. 
 Finally, Section \ref{sec:conclusion} concludes with
limitations and open directions. 

\section{Related work} \label{related-work}
%

The $\COCO$problem has been investigated in a broad line of work on
online convex optimization with long-term, time-varying, and adversarial
constraints.  Since the constraints are chosen by an adversary before the learner selects its action at each round, the learner is not
required to satisfy every constraint at every time; but is required to control the
cumulative violation of the constraints.  Early primal-dual and Lyapunov-based
methods showed that one can trade regret against cumulative constraint violation ($\CCV$)
under various assumptions on the constraint sequence, such as fixed constraints,
stochastic constraints, Slater-type conditions, or weaker adversarial models
\cite{mahdavi2012trading,jenatton2016adaptive,neely2017online,yuan2018online,
yi2021regret,guo2022online,yi2023distributed}.  These works typically do not exploit the
geometry of the nested feasible sets to the fullest extent that is the central feature of this problem and result in sub-optimal $\CCV$ guarantees. 

The adversarial $\COCO$formulation studied in this paper was brought into sharp focus by
Sinha and Vaze~\cite{sinha2024optimal}.  
They proposed a first-order
algorithm with a surrogate cost function that linearly combines the cost and constraint functions. The construction of the surrogate is  based on Lyapunov drift arguments and obtained \(O(\sqrt T)\) regret with
\(\widetilde O(\sqrt T)\) CCV for general convex losses, and $O(\log T)$ regret
with \(\widetilde O(\sqrt T)\) $\CCV$ for strongly convex losses. Their analysis
also yields a logarithmic $\CCV$ in the strongly convex setting under an additional
nonnegative-regret condition.  Their
algorithm projects the iterates onto the original decision set \(\mathcal{X}\) and incorporates
constraints through a Lyapunov-weighted surrogate loss.  On the contrary, the nested-projection method studied here projects the iterates directly onto the time-varying feasible sets
\[
        S_t
        =
        \mathcal{X}\cap\bigcap_{\tau\le t}\{x:g_\tau(x)\le0\}, ~~ t \ge 1.
\]
In particular, since their algorithm only utilizes the gradient information of the constraint functions, it does not fully exploit the geometry of the constraint sets. We fill this gap with our movement lemma (Lemma \ref{lem:robust-nested-projection}). 

A subsequent line of work initiated by Vaze and Sinha~\cite{vaze2026sqrt}
showed that the nested geometry of the sets \(S_t\) can lead to substantially
smaller instance-dependent $\CCV$.  The resulting violation bound depends on geometric features of
the sequence \(S_1\supseteq S_2\supseteq\cdots\), and can be \(O(1)\) for
special structured instances.  This work was one of the first to show that
$\CCV$ should not be viewed merely as another regret term for the constraint
functions \(g_t\); rather, the nested feasible-set geometry can be exploited
directly.  However, their guarantee is instance-dependent, and in the worst case
the bound does not yield the logarithmic strongly convex result obtained here.

Balasundaram, Mahendran, and Vaze~\cite{balasundaram2026breaking} pushed this
geometric viewpoint further in dimension \(d=2\).  They proved that the
projection-based algorithm of Vaze and Sinha achieves \(O(\sqrt T)\) regret and
\(O(T^{1/3})\) CCV for general convex losses in the plane, thereby breaking the
previous \(\widetilde O(\sqrt T)\) violation barrier.  Their proof uses planar
cap geometry: when projecting from \(S_{t-1}\) onto \(S_t\), either the area or
the perimeter of the current feasible set decreases sufficiently.  They also
explicitly conjectured that in \(d=2\), the same projection-based algorithm may
in fact have \(O(1)\) CCV while maintaining $O(\sqrt{T})$ regret, and that a more refined amortization of projection
vectors over time might prove this.  Our lower-bound section disproves this conjecture:
already in dimension two, our constructed instance forces time-growing $\CCV$ of $\frac{(\log T)^{1/3}}{\log \log T}$ for
any no-regret online algorithm. 

A related but distinct problem is Convex Optimization with Nested Evolving
Feasible Sets (CONES), studied by
Karthick Krishna, Balasundaram, and Vaze~\cite{m2026convexoptimizationnestedevolving}.
In CONES, the loss function \(f\) is fixed and known, and the learner
observes the nested feasible set \(S_t\) \emph{before} choosing a feasible action
\(x_t\in S_t\).  The performance metric is regret with respect to the final
hindsight optimum together with total movement cost.  For  fixed strongly convex loss function, they presented an algorithm that achieves nonpositive regret and
\(O(\log T)\) movement cost. In comparison, we work in a significantly more general setting - our loss function can be adversarially time-varying and at each round, the learner chooses an action \emph{before} observing the loss and the constraint function for that round. In this setting, we establish an $O(\log T)$ bound for the $\Regret$ and $O(\log T)$ bound for the movement cost.  


\section{Preliminaries} \label{prelims}
Our analysis for the \COCO problem makes use of a recent deep result on self-contracted polygonal lines induced by an arbitrary norm \cite{stepanov2017self}. 
\begin{definition}[Self-contracted Polygonal Line] \label{self-contraction-def}
	Let $E$ be a metric space equipped with an arbitrary norm $\lnorm \cdot \rnorm.$ A vector $(A_1, A_2, \ldots A_r) \in E^r$ is called \emph{self-contracted} with respect to the norm $\lnorm \cdot \rnorm$ if 
	\begin{eqnarray} \label{self-contr-prop}
	\lnorm A_k - A_j \rnorm \leq \lnorm A_k - A_i\rnorm, ~~~ \forall i \leq j \leq k. 
		\end{eqnarray}
\end{definition}
In other words, as we move forward along the curve, the distance to any fixed future point never increases. The notion of self-contracted curves appears in convex optimization, gradient flows, and metric geometry because trajectories of certain descent dynamics satisfy this property \cite{daniilidis2010asymptotic, longinetti2016bounding}. In our particular application, we will take $E$ to be the $d$-dimensional Euclidean space and work with a non-standard norm that will be defined later in the sequel. The following theorem constitutes our main technical tool.

\begin{theorem}[\cite{stepanov2017self}] \label{self-contraction-thm}
	Let the vector $(A_1, A_2, \ldots, A_r) \in \mathbb{R}^{d}$ be self-contracted with respect to any norm $\lnorm \cdot \rnorm,$  then 
	\[ \sum_{i=1}^{r-1} \lnorm A_{i+1}-A_i \rnorm_2 \leq C \lnorm A_r-A_1 \rnorm_2\] 
	for some $C>0$ depending only on the norm $\lnorm \cdot \rnorm $ and on the space dimension $d$.  
\end{theorem}

\begin{remark}
	It is important to note that although the vector $(A_1, A_2, \ldots, A_r)$ is self-contracted with respect to an arbitrary norm $\lnorm \cdot \rnorm,$ the length of the curve is measured with respect to the standard Euclidean norm $\lnorm \cdot \rnorm_2$. This fact will be exploited in our argument. 
\end{remark}

We will also make use of the standard Projection theorem from convex analysis, which we quote below for easy reference \cite{bertsekas2015convex}.

\begin{theorem}[Projection theorem] \label{proj-thm}
	Let $C$ be a nonempty closed convex subset of $\mathbb{R}^d,$ and let $z$ be a vector in $\mathbb{R}^d.$ There exists a unique vector that minimizes $\lnorm z-x \rnorm_2$ over $x\in C,$ called the projection of $z$ on $C,$ denoted by $\Pi_C(z).$ Furthermore, a vector $x^*$ is the projection of $z$ on $C$ if and only if 
	\[(z-x^*)'(x-x^*) \leq 0, \quad \forall x \in C.\]
	As an immediate corollary of the above result, we have the Pythagorean theorem 
	\[ \lnorm z-x \rnorm _2 \geq \lnorm x^*-x \rnorm_2, \quad \forall x \in C.\]
    and the non-expansiveness property of the projection operator:
    \[ \lnorm \Pi_C(x) - \Pi_C(y) \rnorm_2 \leq \lnorm x-y\rnorm_2, \quad \forall x, y \in \mathbb{R}^d .\]
	\[ \]
\end{theorem}
See \cite[Proposition 1.1.9, Appendix B, and  Proposition 3.2.1]{bertsekas2015convex} for the proof. 
\section{COCO via Nested Projected OGD}
\label{sec:nested-projected-ogd}
We consider a family of algorithms, called \textsf{Nested Projected Online Gradient Descent} (\texttt{NP-OGD}). At each round, an \texttt{NP-OGD} algorithm first takes a gradient step on the last revealed cost function and then projects the result onto the intersection of all feasible sets revealed so far (See Algorithm \ref{alg:nested-projected-ogd} for a Pseudocode). The \texttt{NP-OGD} algorithms only differ in their choice of the step-sizes.   



To be precise, we define the nested feasible sets 
\[
        S_t
        := \bigcap_{\tau=1}^t \{x \in \mathcal X:g_\tau(x)\le 0\}, \quad t \geq 1.
\]
Thus the set $S_t$ denotes the set of feasible actions with respect to the constraints revealed up to round $t.$
 At the beginning
of round \(t\), the learner plays \(x_t\). It then
observes the cost function \(f_t\) and the constraint function \(g_t\). The next action $x_{t+1}$ is obtained by taking a step along the negative gradient of $f_t$ and then projecting it onto the set $S_t.$

\begin{algorithm}[H]
\caption{Nested Projected-OGD for $\COCO$}
\label{alg:nested-projected-ogd}
\begin{algorithmic}[1]
\REQUIRE Initial point \(x_1\in \mathcal{X}\), step size sequence $\{\eta_t\}_{t\geq 1}$.
\STATE Set \(S_0\gets \mathcal{X}\).
\FOR{\(t=1,2,\ldots,T\)}
    \STATE Play \(x_t\)
    \STATE Observe \(f_t\) and \(g_t\).
    \STATE Refine the current estimate of the feasible set
    \[
        S_t
        \gets
        S_{t-1}\cap\{x\in X:g_t(x)\le 0\}.
    \]
    \STATE Take a step along the negative subgradient direction \(h_t\in\partial f_t(x_t)\) and project the result onto $S_t$
    \[
        x_{t+1}
        =
        \Pi_{S_t}\left(x_t-\eta_t h_t\right).
    \]
\ENDFOR
\end{algorithmic}
\end{algorithm}

The algorithm is online: the action at round $t,$ \(x_t\), depends only on losses and constraints
revealed before \(t\).  The new loss \(f_t\) and constraint \(g_t\) are
used only to compute \(x_{t+1}\), which is played from the next round onward.
\begin{remark}[Per-round Computational Cost]
Algorithm~\ref{alg:nested-projected-ogd} only needs one first-order query
\(h_t\in\partial f_t(x_t)\) and one Euclidean projection onto the accumulated
feasible set \(S_t\). For many structured constraints (e.g., half-spaces or $2$-norm balls), this operator can be implemented highly efficiently via standard active-set methods.

\end{remark}
\subsection{Bounding the Movement Cost for \textsf{NP-OGD}} \label{analysis}

We establish a general result for the iteration \eqref{pgd-generic} where the vector $e_t$ at each step $t \geq 1$ could be arbitrary (in particular, not necessarily related to the gradients of any function). We show that if a pure nested projection curve is additively perturbed by
vectors $\{e_t\}_{t \geq 1}$, then its total movement cost is bounded by the diameter of the set $\mathcal{X}$ plus the total perturbation size, up to a
multiplicative constant which depends only on the dimension $d$. Importantly, our upper bound for the movement cost does not explicitly depend on the horizon length $T.$ 
To build intuition, consider the unperturbed case where $e_t=0, \forall t.$ In this setting, the trajectory itself is self-contracted and Theorem \ref{self-contraction-thm} immediately yields a constant movement bound independent of $T$. Lemma \ref{lem:robust-nested-projection} shows that introducing perturbations increases the total movement by at most an additive term proportional to the cumulative perturbation magnitude.


\begin{lemma}[Bounding the Movement Cost]
\label{lem:robust-nested-projection}
Let
\[
        K_0\supseteq K_1\supseteq\cdots\supseteq K_{T}
\]
be a nested sequence of nonempty closed convex subsets of \(\mathbb R^d\), such that \(\text{diam}(K_0)=D\).  Let \(x_1\in K_0\), and suppose that the successive iterates are obtained as
\begin{eqnarray} \label{iters}
        x_{t+1}
        =
        \Pi_{K_t}(x_t - e_t),
        \qquad
        1 \leq t \leq T,
\end{eqnarray}
where $\big(e_t\in\mathbb R^d, 1\leq t \leq T\big)$ is an arbitrary sequence of vectors and $\Pi_{K_i}(\cdot)$ denotes the standard Euclidean projection operation onto the closed and convex set $K_i, 1\leq i\leq T$.  Then the cumulative movement of the successive iterates \eqref{iters} can be upper-bounded as
\begin{eqnarray} \label{movement-cost-ub}
        \sum_{t=1}^{T}\|x_{t+1}-x_t\|_2
        \le
        C_{d}
        \left(
            D+\sum_{t=1}^{T}\|e_t\|_2
        \right),
\end{eqnarray}
where the constant $C_{d} >0$ depends only on the dimension $d.$
\end{lemma}

\paragraph{Proof outline:} Our high-level strategy is to make use of the finite length property of the self-contracted polygonal line (Theorem \ref{self-contraction-thm}). However, due to movement along the vectors $\{e_t\}_{t\geq 1},$ the iterates $\{x_t\}_{t\geq 1}$ no longer satisfy the self-contraction property \eqref{self-contr-prop}. Interestingly, we show that this issue can be mitigated by appropriately \emph{lifting} the iterates to a higher dimension. Specifically, the new iterates are self-contracted with respect to a particular norm $\lnorm \cdot \rnorm_{\oplus}$ defined later. The movement bound for the original iterates then follows from the movement bound of the lifted sequence. We give the detailed proof below.  


\begin{proof}
Fix any round $t \geq 1.$ By the non-expansiveness property of the Euclidean projection (Theorem \ref{proj-thm}), we have 
\begin{eqnarray} \label{perturb1}
        \|x_{t+1}-\Pi_{K_t}(x_t)\|_2
        =
        \left\|
        \Pi_{K_t}(x_t - e_t)-\Pi_{K_t}(x_t)
        \right\|_2
        \le
        \lnorm e_t \rnorm_2,	
\end{eqnarray}
Now fix any round \(k>t\).  Since $x_k\in K_{k-1}\subseteq K_t$, using Pythagorean theorem for Euclidean projection, we have 
\begin{eqnarray} \label{pythagorean1}
	        \|\Pi_{K_t}(x_t)-x_k\|_2 \le \|x_t-x_k\|_2.
\end{eqnarray}
Therefore,
\begin{eqnarray*}
        \|x_{t+1}-x_k\|_2 &=& \lnorm x_{t+1} - \Pi_{K_t}(x_t) + \Pi_{K_t}(x_t) - x_k \rnorm_2 \\
        &\stackrel{(a)}{\le}& 
        \|x_{t+1}-\Pi_{K_t}(x_t)\|_2+ \|\Pi_{K_t}(x_t)-x_k\|_2 \\
        &\stackrel{(b)}{\le}&
        \|x_t-x_k\|_2+\lnorm e_t\rnorm_2,
\end{eqnarray*}
where (a) follows from Triangle inequality for the Euclidean norms and (b) follows from Eqns \eqref{perturb1} and \eqref{pythagorean1}. Iterating this inequality, we have that for any \(i<j<k\) the following approximate self-contraction property:
\begin{eqnarray} \label{approx-contraction}
        \|x_j-x_k\|_2
        \le
        \|x_i-x_k\|_2
        +
        \sum_{r=i}^{j-1}\lnorm e_r \rnorm_2 .
\end{eqnarray}
The above inequality shows that the sequence $\{x_i\}_{i \geq 1}$ satisfies an \emph{approximate} version of the self-contraction property, given by Eqn.\ \eqref{self-contraction-def}, up to an additive perturbation. We now use a \emph{lifting} argument to lift the sequence to a higher dimension so that it \emph{exactly} satisfies the self-contraction property, and hence, Theorem \ref{self-contraction-thm} can be used (See Figure \ref{lifting}).  For this, define the tail perturbation sequence $\{R_t\}_{t=1}^{T}$ as 
\[
        R_t:=\sum_{r=t}^{T} \lnorm e_r \rnorm_2, \qquad 1\le t \le T, \quad R_{T+1}=0.
\]
 Next, we define a sequence of $d+1$ dimensional vectors $\{A_t\}_{t=1}^{T}$ by appropriately lifting the action sequence to \(\mathbb R^{d+1}\): 
 \[
        A_t:=(x_t,R_t),
        \qquad 1\leq t \leq T+1.
\]
Note that the above sequence $\{A_t\}_{t}$ is constructed in a non-causal fashion. However, this does not pose any issue as the sequence $\{A_t\}_{t}$ will be used only in the proof. 

Next, we equip the space \(\mathbb R^{d+1}\) with a non-standard norm $\lnorm \cdot \rnorm_{\oplus}$ defined below. Let $\bm{x} \equiv (x_1, x_2, \ldots, x_d, x_{d+1}) = (\bm{u}, x_{d+1})$ be an arbitrary vector in $\mathbb{R}^{d+1},$ where $\bm{u}\equiv(x_1, x_2, \ldots, x_d)$ is the $d$-dimensional vector comprising the first $d$ components of $\bm{x}.$ We now define the norm of the vector $\bm{x}=(\bm{u}, x_{d+1})$ as:
\[
        \|\bm{x}\|_{\oplus}:=\|\bm{u}\|_2+|x_{d+1}|.
\]
It can be easily checked that $\lnorm \cdot \rnorm_{\oplus}$ satisfies all three axioms, namely Positive Definiteness, Homogeneity, and Triangle inequality, required of a norm. Now, for any triple \(i<j<k\), we have 
\[
\begin{aligned}
        \|A_j-A_k\|_{\oplus}
        &=
        \|x_j-x_k\|_2+R_j-R_k  \\
        &\stackrel{(a)}{\le}
        \|x_i-x_k\|_2+(R_i-R_j)+(R_j-R_k) \\
        &=
        \|x_i-x_k\|_2+R_i-R_k \\
        &=
        \|A_i-A_k\|_{\oplus},
\end{aligned}
\]
where in step (a), we have used the approximate self-contraction property given by Eqn.\ \eqref{approx-contraction}. 
Using Definition \ref{self-contraction-def}, the above inequality implies that the sequence $(A_1, A_2,\ldots, A_{T+1})$ is a self-contracted polygonal line in $\mathbb{R}^{d+1}$ with respect to the norm $\lnorm \cdot \rnorm_{\oplus}$. Hence, using Theorem \ref{self-contraction-thm}, we conclude that 
\begin{eqnarray} \label{length-bd}
		 \sum_{t=1}^{T} \lnorm A_{t+1}-A_t \rnorm_2 \leq C_d \lnorm A_{T+1}-A_1 \rnorm_2
		 \end{eqnarray}
	for some $C_d>0,$ which depends only on the dimension $d$ and on the norm $\lnorm \cdot \rnorm_{\oplus}$. Since the norm $\lnorm \cdot \rnorm_{\oplus}$ is fixed throughout the analysis, we suppress this dependence. By definition, we have
\[
        \|A_{t+1}-A_t\|_{2}
        \ge
        \|x_{t+1}-x_t\|_2.
\]
Furthermore, we can bound 
\[ \lnorm A_{T+1}-A_1\rnorm_2^2 = ||x_{T+1}-x_1||_2^2 + (R_{T+1}-R_1)^2 \stackrel{(a)}{\leq} D^2 + R_1^2.\]
	where in (a), we have used the finite diameter of the set $K_0$. Combining the above two estimates with \eqref{length-bd}, we conclude that 
	\begin{eqnarray*}
		 \sum_{t=1}^{T} \|x_{t+1}-x_t\|_2 \leq C_d (D + \sum_{t=1}^{T} \lnorm e_t \rnorm_2),
	\end{eqnarray*}
	as claimed.
	\begin{figure}
	\centering	
	\includegraphics[scale=0.7]{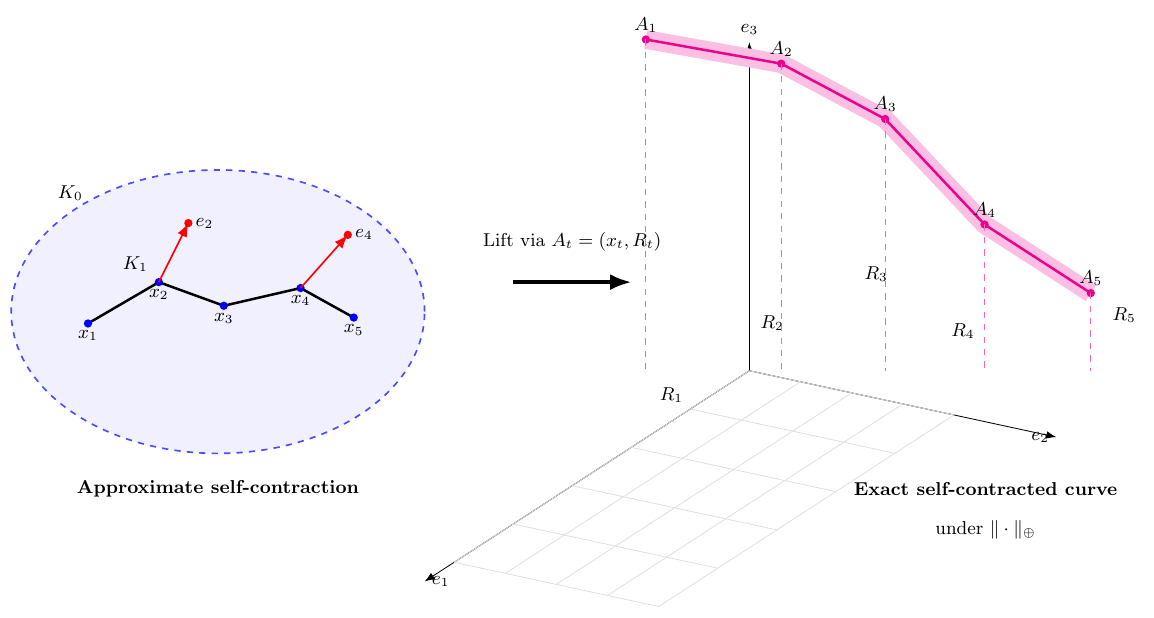} 
	\put(-350,75){\scriptsize{Original trajectory}}
		\put(-80,75){\scriptsize{Lifted trajectory}}
	\caption{\small{Illustration of exact Self-Contraction via the lifting argument}}
	\label{lifting}
	\end{figure}
%
%
\end{proof}

\subsection{$\Regret$ and $\CCV$ bounds for $\texttt{NP-OGD}$}
In this section, we bound the $\Regret$ and $\CCV$ of the $\textsf{NP-OGD}$ algorithm for both convex and strongly convex losses. While the proof of the $\Regret$ bound makes use of standard arguments, the proof of the $\CCV$ bound utilizes Lemma \ref{lem:robust-nested-projection}. 
\begin{lemma}[$\Regret$ bound]
\label{lem:nested-ogd-regret}
Suppose that each loss function \(f_t\) is \(H_t\)-strongly convex and \(G\)-Lipschitz on \(\mathcal{X}\),
where we set \(H_t=0\) for
convex yet not strongly convex losses.  Then the regret of Algorithm \ref{alg:nested-projected-ogd} can be bounded as:
\[
\begin{aligned}
      \mathsf{Regret}_T(x^*) := \sum_{t=1}^T f_t(x_t)
        -
        \sum_{t=1}^T f_t(x^\star)
        \le
        \frac{D^2}{2}
        \sum_{t=1}^T
        \bigl(\frac{1}{\eta_t}-\frac{1}{\eta_{t-1}}-H_t\bigr)_+
        +
        \frac{G^2}{2}
        \sum_{t=1}^T\eta_t, \quad \forall x^* \in S_T,
\end{aligned}
\]
 where we have defined $
\frac{1}{\eta_0} \equiv 0.$ 

\end{lemma}

The proof of Lemma \ref{lem:nested-ogd-regret} follows from an important yet minor variation of the standard arguments and is given in Appendix \ref{reg-bd-proof} for reference.
In our next result, we upper bound the $\CCV$ by the movement cost. 
\begin{lemma}[$\CCV$ bound]
\label{lem:nested-ogd-movement}
Assume each \(f_t\) is convex and \(G\)-Lipschitz and each \(g_t\) is quasi-convex and \(G\)-Lipschitz on $\mathcal X$.
Then the $\CCV$ of Algorithm~\ref{alg:nested-projected-ogd} can be bounded as:
\begin{eqnarray*}
       \CCV_T := \sum_{t=1}^T (g_t(x_t))^+
        \le  G\sum_{t=1}^T\|x_{t+1}-x_t\|_2 
        \leq
        GC_{d}
        \left(
            D+ G \sum_{t=1}^T \eta_t
        \right),
\end{eqnarray*}
where $C_d$ is a constant that depends only on the dimension $d.$
\end{lemma}

\begin{proof}
Recall that the update of Algorithm~\ref{alg:nested-projected-ogd} is exactly
\[
        x_{t+1}
        =
        \Pi_{K_t}(x_t - e_t).
\]
with $e_t=\eta_t h_t,$ where $h_t \in \partial f_t(x_t).$ 
Thus $        \lnorm e_t\rnorm_2 \le \eta_t G,$ where we have used the $G$-Lipschitzness of the cost functions.
Hence, applying Lemma~\ref{lem:robust-nested-projection} with $K_t=S_t, t \geq 1,$ we obtain the following movement cost bound:
\begin{eqnarray} \label{movement-cost-bound}
        \sum_{t=1}^T\|x_{t+1}-x_t\|_2
        \le
        C_{d}
        \left(
            D+ G \sum_{t=1}^T \eta_t
        \right).
\end{eqnarray}

%
%
Finally,  by construction,
\[
        x_{t+1}\in S_t\subseteq\{x:g_t(x)\le0\}.
\]
Therefore,
\[
        g_t(x_{t+1})\le0.
\]
Since \(g_t\) is \(G\)-Lipschitz, the incremental violation may be bounded as:
\[
        g_t(x_t)
        \le
        g_t(x_{t+1})+G\|x_t-x_{t+1}\|_2
        \le
        G\|x_t-x_{t+1}\|_2.
\]
Summing over \(1\leq t\leq T\) yields
\begin{eqnarray*}
\CCV_T
        =
        \sum_{t=1}^T(g_t(x_t))^+
        \le
        G\sum_{t=1}^T\|x_{t+1}-x_t\|_2  
        \stackrel{\text{Eq.} \eqref{movement-cost-bound}}{\le} GC_{d}
        \left(
            D+ G \sum_{t=1}^T \eta_t
        \right).
\end{eqnarray*}
\end{proof}
We now state our main result of the paper by considering specific step size sequences suitable for different classes of loss functions. 

\begin{theorem}[Performance bounds for \texttt{NP-OGD}] \label{thm:main}
1. Assume each loss function \(f_t\) to be convex and \(G\)-Lipschitz. If Algorithm \ref{alg:nested-projected-ogd} is run with the step size sequence $\eta_t= \frac{D}{G \sqrt{t}}, t\geq 1,$ it achieves the following $\Reg$ and $\CCV$ bounds:

\[
    \Reg_T \le \frac32GD\sqrt T, \quad \CCV_T \le C_{d}G \left(D+2D\sqrt T \right).
\]

2. Assume each loss function \(f_t\) to be \(\mu\)-strongly convex and \(G\)-Lipschitz. If Algorithm \ref{alg:nested-projected-ogd} is run with the step size sequence $\eta_t= \frac{1}{\mu t}, t\geq 1,$ it achieves the following $\Reg$ and $\CCV$ bounds:
\begin{eqnarray*}
    \Reg_T \le \frac{G^2}{2\mu}(1+\log T), \quad 
    \CCV_T \le C_{d}G \left( D+\frac{G}{\mu}(1+\log T) \right).
\end{eqnarray*}
In the above, $C_d$ is a constant that depends only on the dimension $d.$
\end{theorem}
It is well-known that the above regret bounds are minimax optimal even in the unconstrained setting \cite{hazan2022introduction}. The proof of Theorem \ref{thm:main} follows immediately from Lemma \ref{lem:nested-ogd-regret} and \ref{lem:nested-ogd-movement}
upon substituting the respective expressions for the step size sequence. 

\subsection{An Application: Logarithmic Guarantees for Strongly Convex Constraints} \label{structure-section}
In our analysis so far, we have only used the quasi-convexity of the constraint functions, which only requires the sub-level sets of the constraint functions to be convex. In particular, we have not exploited any detailed convex-analytic properties of the constraint functions.
As an application of Theorem \ref{thm:main}, we show that when the constraint functions have a strictly positive curvature, the $\Regret$ and $\CCV$ guarantees can be improved beyond the baseline $O(\sqrt{T})$ bounds while assuming only the convexity for the loss functions and suitably restricting the class of adversaries. Recall that the constraint function at round $t$ is denoted by $g_t$ which encodes a constraint set $\{x: g_t(x) \leq 0\}.$ We make the following two assumptions.

\begin{assumption}[Strongly convex constraints] \label{assump-norm}
We assume that each constraint function $g_t$ is $\mu$-strongly convex and $G$-Lipschitz for some constants $\mu, G > 0.$ 
\end{assumption}
Recall that the $\CCV$ definition given by Eqn.\ \eqref{reg-ccv-def}, considers only non-negative violations $g_t(x_t)_+$ at every round. However, to exploit the strong convexity of the constraint functions, we now define the constraint violation at round $t$ to be $g_t(x_t).$ Hence, the violation at any round could now be either positive or negative, depending on whether the current action $x_t$ satisfies the online  constraint or not. Thus, specific to this section, we work with the following weaker definition of $\CCV,$ called \textsf{Signed Cumulative Violation}: 
\begin{eqnarray} \label{ccv-def-new}
\mathsf{SCV}_T = \sum_{t=1}^T g_t(x_t).
\end{eqnarray}
 We keep the regret definition the same as Eqn.\ \eqref{reg-ccv-def}. Hence, unlike before, positive violations at some rounds can be (partly) compensated by negative violations at some other rounds. Definition \eqref{ccv-def-new} is applicable to the case where one is only interested in bounding the long-term $\CCV$ over the entire horizon of length $T.$ We now make the following assumption on the adversary. 
\begin{assumption}[Non-Negative $\mathsf{SCV}$] \label{assump-non-neg-CCV}
	The adversary ensures that the cumulative constraint violation at the end of the horizon is non-negative, \emph{i.e.,} $\mathsf{SCV}_T = \sum_{t=1}^T g_t(x_t) \geq 0.$
\end{assumption}
Assumption \ref{assump-non-neg-CCV} is satisfied, \emph{e.g.,} by an 
\emph{Agile Adversary} that chooses the constraints adaptively so that the learner always incurs a non-negative violation at each round, \emph{i.e.}, $g_t(x_t) \geq 0, \forall t.$

We now consider the $\COCO$problem with convex cost functions such that the online constraints satisfy Assumption \ref{assump-norm} and Assumption \ref{assump-non-neg-CCV}. In Algorithm \ref{alg:structured-COCO} described below, we invoke $\texttt{NP-OGD}$ with a modified cost function that achieves $O(\log T)$ $\Regret$ and $O(\log T)$ $\mathsf{SCV}$.

\begin{algorithm}[H]
\caption{$\mathsf{COCO}$ with Strongly Convex Constraints}
\label{alg:structured-COCO}
\begin{algorithmic}[1]
\FOR{\(t=1,2,\ldots,T\)}
\STATE Play $x_t$ using Algorithm \ref{alg:nested-projected-ogd}
\STATE Observe the cost function $f_t$ and the constraint function $g_t.$
\STATE Compute surrogate cost $\hat{f}_t(x) = f_t(x)+g_t(x).$
\STATE Pass the surrogate cost $\hat{f}_t$ and the constraint $g_t$ to  Algorithm \ref{alg:nested-projected-ogd}. 
   \ENDFOR
\end{algorithmic}
\end{algorithm}

\subsection{Analysis of Algorithm \ref{alg:structured-COCO}}
By Assumption \ref{assump-norm}, the surrogate cost functions are $\mu$-strongly convex and $2G$-Lipschitz. Hence, using Theorem \ref{thm:main}, part (2), we conclude that for any feasible comparator $x^* \in \mathcal{X}^*,$ we have 
\begin{eqnarray*}
	\Regret_T^{\hat{f}}(x^*) = O(\log T), ~~~ \CCV_T = O(\log T). 
\end{eqnarray*}
Now note that 
\begin{eqnarray} \label{reg-decomp}
	\Regret_T^{\hat{f}}(x^\star) = \Regret_T^{f}(x^*) + \underbrace{\sum_{t=1}^T \big(g_t(x_t) - g_t(x^*)\big)}_{(A)}. 
\end{eqnarray}
By Assumption \ref{assump-non-neg-CCV}, we have that $ \sum_{t=1}^T g_t(x_t) \geq 0.$ Also, from the feasibility property of the benchmark, we have $g_t(x^*) \leq 0, \forall t \geq 1.$ Thus term (A) is non-negative. Hence, from Eqn.\ \eqref{reg-decomp}, we conclude that Algorithm \ref{alg:structured-COCO} yields 
\begin{eqnarray*}
	\Regret_T(x^*) \leq \Regret_T^{\hat{f}}(x^*) = O(\log T), ~~ \textrm{and} ~~~ \mathsf{SCV}_T = O(\log T).
\end{eqnarray*}
\paragraph{Discussion:} At the outset, it might seem surprising that Algorithm \ref{alg:structured-COCO} achieves $O(\log T)$ regret for convex cost functions, apparently violating the well-known $\Omega(\sqrt{T})$ regret lower bound \cite{hazan2022introduction}. To understand why this faster rate is possible, note that the regret benchmark $x^\star$ is not arbitrary - it is restricted to be a point that satisfies all constraints. The agility of the adversary ensures that the constraints are enforced in the long-term. Without Assumption \ref{assump-non-neg-CCV}, the adversary is free to produce the same constraint at every round - effectively reducing the problem to an unconstrained OCO, for which regret is indeed lower bounded by $\Omega(\sqrt{T}).$

\subsection{Lower bounds for $\COCO$with Strongly Convex Loss} \label{lower_bound_section}

Since the $\COCO$problem is a generalization of the classical OCO problem (this follows immediately upon choosing the constraint function $g_t =0, \forall t \geq 1$), the minimax $\Regret$ for $\COCO$is naturally lower bounded by that of the standard unconstrained OCO problem (\emph{i.e.,} $\Regret \geq \Omega(\log T)$) \cite{hazan2022introduction}. This motivates us to formulate the converse problem as lower bounding the $\CCV$ for all online policies that satisfy a sublinear regret upper bound of $R(T)$ against all feasible actions\footnote{By sublinear, we mean $R(T) \leq c T^{1-\delta}$ for some constants $c >0$ and $ 0< \delta \leq 1.$}.  Formally, for a given sublinear regret budget \(R(T)\), we define the minimax $\CCV$ value as 
\begin{align}
\label{reg-bud-lb}
        V(R(T),T)
        :=
         \inf_{\mathcal A:\sup_I \Reg_T(\mathcal A, I)\le R(T)} \sup_I
         \CCV_T(\mathcal A,I),
\end{align}
where $I$ is an arbitrary problem instance and the outer infimum is taken over all online policies $\mathcal{A}$ satisfying the regret budget $R(T)$ over all valid instances. In this section, our goal is to establish a lower bound for $V(R(T), T)$.

\begin{theorem}[$\CCV$ lower bound]
\label{thm:cone-cap-lower-bound}
Consider the above setup. Fix any online $\COCO$algorithm achieving a sublinear regret of \(R(T) = O(T^{1-\delta})\) for some fixed $0< \delta \leq 1.$ Then there exists a $\COCO$problem instance on the $d$-dimensional unit ball with $\nicefrac{1}{2}$-strongly convex $\nicefrac{3}{2}$-Lipschitz losses and $1$-Lipschitz affine constraints for which, the minimax $\CCV$, defined in Eqn.\ \eqref{reg-bud-lb}, is lower bounded as: \[V(R(T), T) \ge\Omega\bigg(\frac{(\ln T)^{\frac{d-1}{d+1}}}{\ln \ln T}\bigg).\]
%
\end{theorem}

In order to establish the above result, given any $\COCO$algorithm with a regret bound of $R(T)$, we construct a sequence of cost and constraint functions and lower bound the resulting $\CCV.$ 

The construction involves a maximal $\Delta$-packing of the unit sphere $S^{d-1}.$ Specifically, let $U = \{u_1, u_2, \ldots, u_N\}$ be a set of $d$-dimensional unit vectors satisfying 
\[ \lnorm u_i - u_l \rnorm_2 \ge \Delta, ~~ \forall i \neq l.\]
By the standard packing bound on the unit sphere \cite[Section V]{shannon1959probability}, for every sufficiently small
\(\Delta>0\), there exists a set $U$ with 
\[
        |U|= N\ge c_d\Delta^{-(d-1)},
\]
where the constant \(c_d>0\) depends only on the dimension \(d\). 

\paragraph{Construction of the adversarial instance:} We first specify the constraint functions used in our construction. Each constraint function is indexed by a unit vector $u_i \in U$ and a \emph{depth} parameter $k$ where $k \in [L]$ with  $L:=\left\lfloor\frac{1}{4\varepsilon}\right\rfloor$ for some $0 < \varepsilon < 1.$ Specifically, the $j$\textsuperscript{th} constraint function corresponding to the tuple $(k, u_i) \equiv j$ is given by: 
\begin{eqnarray} \label{constr-fun}
g(x) \equiv \langle u_i, x \rangle - (1-k\varepsilon).	
\end{eqnarray}
 This constraint function is affine, $1$-Lipschitz and corresponds to the following constraint
\begin{eqnarray} \label{constr-set}
	 \langle u_i, x \rangle \leq 1 - k\varepsilon, ~~ x \in \mathcal{X}. 
\end{eqnarray}

These constraints appear in \emph{lexicographic order} in the tuple $(k,i)$ at some point in time in the horizon, to be specified below. To elaborate, first, all depth-$1$ constraints ($k=1$) appear (there are exactly $N$ of them) and then all $N$ depth-$2$ constraints ($k=2$) appear etc. The successive constraints result in the following shrinking sequence of feasible sets 
\[\mathcal{X} \equiv K_0 \supseteq K_1 \supseteq K_{M+1},\]
where $K_j = K_{j-1} \cap \{ x \in \mathcal{X}: \langle u_i, x \rangle \leq 1 - k\varepsilon \},$ where $j= (k, u_i),$ $M = NL \geq c_d' \frac{\Delta^{-(d-1)}}{\varepsilon},$ is the total number of constraints where $c_d'$ is a constant depending only on the dimension $d.$ We now make the following claim.

\begin{figure*}
\centering
	\includegraphics[scale=0.9]{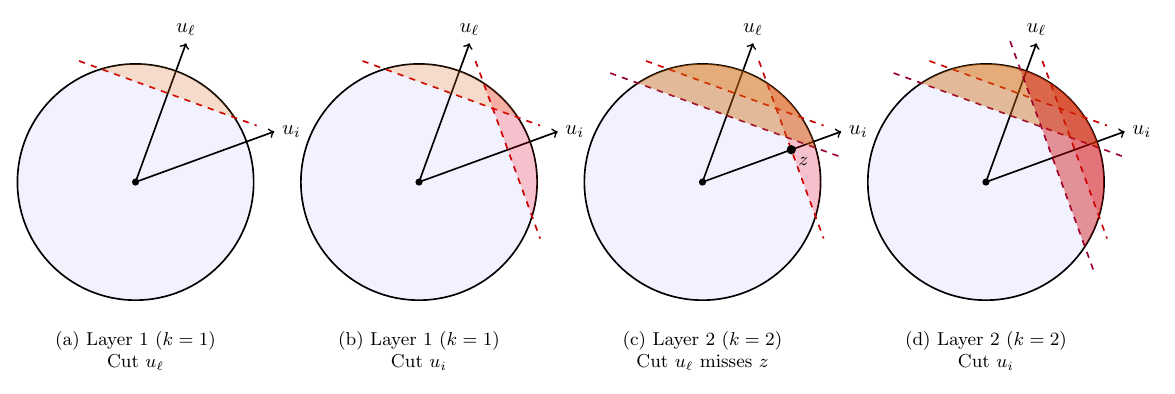}
\caption{\small{Lexicographic layer progression. Notice in (c) how the the deeper cut at $u_\ell$ does not exclude the tip $z$ along the direction $u_i$.}}
\label{fig:layered_cuts}
\end{figure*}

\begin{claim} \label{feas-claim}
With $\Delta = 2 \sqrt{\varepsilon},$ we have $z_j:=(1-(k-1)\varepsilon)u_i \in K_{j-1}.$ 
\end{claim}

\begin{proof}
Let us index $j = (k,u_i).$ 
Clearly, the point $z$ satisfies all previous constraints in the direction of $u_i$ as 
\[\langle u_i, z_j \rangle = 1 - (k-1)\varepsilon \leq 1-p \varepsilon, ~\forall p<k.  \]

Now consider a previous constraint with index $(p, u_l)$ with $u_l \neq u_i$ and $p\leq k.$ 
Since
\[
        \|u_i-u_\ell\|_2^2
        =
        2(1-\langle u_i,u_\ell\rangle),
\]
the \(\Delta\)-separation of the set $U$ gives
\[
        \langle u_i,u_\ell\rangle
        \le
        1-\frac{\Delta^2}{2}.
\]
Moreover, because \(k\le L = \lfloor \frac{1}{4\varepsilon} \rfloor\), we have $(k-1)\varepsilon\le \frac{1}{4}.$
Finally, choosing $\Delta=2\sqrt{\varepsilon},$
we have
\[
\begin{aligned}
        \langle u_\ell,z_j\rangle
        =
        (1-(k-1)\varepsilon)\langle u_\ell,u_i\rangle
        &\le
        (1-(k-1)\varepsilon)
        \left(1-\frac{\Delta^2}{2}\right)\\
        &\le
        (1-(k-1)\varepsilon)(1- 2\varepsilon)\\
        &= 1 - k\varepsilon - \varepsilon(1 - 2 (k-1)\varepsilon) \\
        &\le 1 - k\varepsilon 
        \le 1 - p \varepsilon.
\end{aligned}
\]
Thus the point \(z_j\) is feasible for all previous constraints.  

\end{proof}

\paragraph{A high-level strategy for constructing and scheduling the cost and the constraint functions}

We divide the entire time horizon $T$ into multiple phases, where phase $j$ corresponds to a constraint function $j = (k, u_i)$ given by Eqn.\ \eqref{constr-fun}. 
 At phase \(j \geq 1\), the adversary makes a candidate action 
\(z_j:=(1-(k-1)\varepsilon)u_i \) attractive to the learner by choosing the $\nicefrac{1}{2}$-strongly convex cost function defined in Eqn.\ \eqref{loss-phase-j}
and zero constraint function throughout the phase except the last round of phase $j$. Claim \ref{feas-claim} shows that the action $z_j$ remains feasible until the end of phase $j.$ If the learner avoids playing any actions from the neighborhood of $z_j$ throughout, the
adversary could simply continue the phase, causing the learner to incur large regret.  Therefore a no-regret learner, required by definition \ref{reg-bud-lb}, must play an action from the neighborhood of $z_j$. Immediately after the learner plays an action from the neighborhood of $z_j$, the adversary produces a half-space constraint \eqref{constr-fun} which makes the action $z_j$ infeasible as 
\[\langle u_i, z_j\rangle = 1- (k-1)\varepsilon > 1-k \varepsilon. \]
 Consequently, the learner incurs a violation of at least $\varepsilon$ and phase $j$ ends.
Finally, the total $\CCV$ is obtained by summing up the contributions from all $M \geq \Omega(\varepsilon^{-\nicefrac{1}{2}(1+d)})$ phases. In the following, we make the above construction explicit.

%
%
%
%
%
%
%

\begin{proof}
%
%

Define $m = \lfloor\min \big( c_d' \varepsilon^{-(\frac{d+1}{2})}, \frac{\log \frac{T}{R(T)+1}}{\log (13/\varepsilon)}\big)\rfloor,$ where $c_d'$ is some fixed dimension-dependent constant. 
We now construct a problem instance with \(m\) phases.

 Define the trigger set for the $j$\textsuperscript{th} phase:
\begin{eqnarray} \label{t-cap}
        \mathcal T_j
        :=
        \{x\in \mathcal{X}:\langle u_i,x\rangle\ge 1-(k-\nicefrac12)\varepsilon\}.
\end{eqnarray}
  Phase $j$ begins at round $P_j+1$ (we set  
        $P_1:=0$) and continues for $L_j$ rounds unless it ends prematurely due to the learner playing an action from the trigger set $\mathcal{T}_j.$ Thus $P_{j+1} \leq P_j + L_j.$ During phase
\(j\),
the adversary reveals the strongly convex loss function
\begin{eqnarray} \label{loss-phase-j}
        f_t^{(j)}(x)
        :=
        \frac{1}{4}\|x\|_2^2
        -
        \langle u_i,x\rangle,
\end{eqnarray}
and set the constraint \(g_t\) trivially to zero until the last round of phase $j$. The loss function \(f_t^{(j)}\) is clearly \(\nicefrac{1}{2}\)-strongly convex. Furthermore, its Lipschitz constant on the standard unit ball can be bounded as 
\[
        \|\nabla f_t^{(j)}(x)\|_2
        =
        \left\|x/2-u_i\right\|_2
        \le \lnorm x \rnorm_2/2 + \lnorm u_i \rnorm_2 \leq \nicefrac{3}{2}. 
        \]
Thus all losses are \(\nicefrac32\)-Lipschitz on \(\mathcal{X}\). 
The maximum length $L_j$ of the $j$\textsuperscript{th} phase is taken to be 
\[
        L_j
        :=
        \left\lceil
         \frac{4}{\varepsilon}\big(R(T)+3P_j+1\big).
        \right\rceil.
\]

We now make the following claim.

\begin{claim} \label{claim-trigger}
A no-regret learner must play an action from the trigger set $T_j,$ defined in Eqn.\ \eqref{t-cap}, during phase $j.$
\end{claim} 
\begin{proof}
Suppose, for the sake of contradiction, that the learner did not play any action from the trigger set \(\mathcal{T}_j\) during phase \(j \equiv (k, u_i)\). In this case, the adversary reveals the strongly convex loss function \eqref{loss-phase-j} and the zero constraint function throughout, which keeps action $z_j$ feasible. Then for any round \(t\) in phase \(j\), we have 
\begin{eqnarray} \label{non-enter}
        \langle u_i,x_t\rangle<1-(k-\nicefrac12)\varepsilon.
\end{eqnarray}
For notational convenience, define
\[
        s:=\langle u_i,x_t\rangle, ~~\textrm{and}~~ h_j = 1-(k-1)\varepsilon.
\]
Thus we can write 
\begin{eqnarray} \label{key-eq}
z_j = h_j u_i \implies \lnorm z_j\rnorm = h_j ~~ \textrm{and}~~ s < h_j - \varepsilon/2.	
\end{eqnarray}
Furthermore, using Cauchy-Schwarz inequality, we have \(\|x_t\|_2^2\ge s^2\). Hence, the incremental regret of the learner at any round $t$ is given by 
\begin{eqnarray*}
	 f_t^{(j)}(x_t)-f_t^{(j)}(z_j) &=& \frac{1}{4}\lnorm x_t\rnorm^2 - \langle u_i, x_t \rangle - \frac{1}{4}\lnorm z_j\rnorm^2 + \langle u_i, z_j\rangle \\
	 &\geq & \frac{1}{4}s^2 - \frac{1}{4}h_j^2 -  s + h_j =  (h_j-s) \big(1 - \frac{1}{4}(h_j+s)\big)  \geq \frac{1}{2}(h_j-s) \geq \varepsilon/4,
 \end{eqnarray*}
 where we have made use of Eqn.\ \eqref{key-eq}. 


%
%
%

The previous \(P_j\) rounds before phase $j$ can contribute negative regret relative to the feasible action \(z_j\).
However, note that every previous loss is \(\nicefrac{3}{2}\)-Lipschitz on a set of diameter \(2\).  Hence,
each previous round can contribute at least \(-2 \times \nicefrac{3}{2} = -3\) regret relative to the action $z_j$ per round.  Therefore the total regret
against the feasible comparator \(z_j\) on this instance is at least
\[
        -3P_j+\frac{\varepsilon}{4} L_j.
\]
Hence, as long as the length of the $j$\textsuperscript{th} phase is at least  
\[
         L_j
        \ge
        \frac{4}{\varepsilon}\big(R(T)+3P_j+1\big),
\]
the regret of the learner becomes strictly more than $R(T),$
contradicting our assumption on the learner's regret bound.  Thus the learner must play an action from the trigger set $\mathcal{T}_j.$
\end{proof}
\begin{claim} \label{ccv-phase}
	Each phase contributes at least $\nicefrac{\varepsilon}{2}$ to the final $\CCV$.
\end{claim}
\begin{proof}
Let \(\tau_j\) be the first round in phase \(j\) with
$
        x_{\tau_j}\in\mathcal T_j.
$
At this round, the adversary reveals the $j$\textsuperscript{th} constraint function,  given by Eqn.\ \eqref{constr-fun}, ending phase $j$. At all previous rounds in phase $j$, the constraint function was taken trivially equal to zero. 
%
Thus the constraint violation at the final round of phase $j$ is given by 
\[g(x_{\tau_j}) = \langle u_i, x_{\tau_j} \rangle - (1-k\varepsilon) \geq 1 - (k-\frac{1}{2})\varepsilon - (1-k \varepsilon) \ge \varepsilon/2, \]
 where inequality (a) follows from the fact that the action $x_{\tau_j}$ belongs to the trigger set $\mathcal{T}_j$ defined in Eqn.\ \eqref{t-cap}.
\end{proof}

\begin{claim} \label{fits-length}
All \(m\) phases fit into a horizon of length $T.$	
\end{claim}
 \begin{proof} 
 Define
$
        Q_j:=P_j+R(T)+1.
$
The length of phase $j$ can be upper bounded as
\[
\begin{aligned}
        L_j
        &\le
        \frac{4(R(T)+3P_j+1)}{\varepsilon}+1 \leq \frac{12}{\varepsilon}Q_j,
\end{aligned}
\]
where we have used the fact that $0 < \varepsilon <1. $
for a numerical constant \(C>0\).  Hence we have the following recursion
\[
        Q_{j+1} = P_{j+1} + R(T)+1 \leq P_j + L_j + R(T)+1 = Q_j + L_j \leq \frac{13}{\varepsilon}Q_j.
\]
This yields
\[
        Q_{m+1}
        \le
        \big(\frac{13}{\varepsilon}\big)^m Q_1
        =      \big(\frac{13}{\varepsilon}\big)^m  \left(
            R(T)+1
        \right) \stackrel{(a)}{\leq} T.
\]
where inequality (a) follows from the definition of $m.$
Thus \(P_{m+1}\le Q_{m+1}\le T\), so all phases fit.
\end{proof}

Substituting the definition of \(m\) and absorbing numerical constants gives
the claimed lower bound.
Thus the $\CCV$ incurred by any policy is at least $m\varepsilon.$
Finally, for $T \geq 3,$ choosing $\varepsilon^{-1} = (\ln T)^{2/(d+1)},$ and using the fact that $R(T) = O(T^{1-\delta}), 0<\delta \le 1,$  we obtain the following minimax $\CCV$ lower bound
\begin{eqnarray*}
	V(R(T),T) \geq  \Omega \bigg((\ln T)^{\frac{d-1}{d+1}}/ (\ln 13 + \frac{2}{d+1} \ln \ln T)\bigg)  = \Omega\bigg(\frac{(\ln T)^{\frac{d-1}{d+1}}}{\ln \ln T}\bigg).
\end{eqnarray*}
%
\end{proof}

\begin{remark}[Implication for convex-loss regret budgets]
\label{rem:fixed-dimensional-true-lower-bound}
Although Theorem~\ref{thm:cone-cap-lower-bound} is stated for strongly convex
losses, the lower-bound mechanism is not specific to strong convexity.  The same
phase construction can be implemented with convex losses, and the proof depends
on the regret guarantee only through the budget \(R(T)\).  Hence the same
fixed-dimensional lower-bound order persists for any polynomially sublinear
regret budget; in particular, it applies to the usual convex-loss regime
\(R(T)=O(\sqrt T)\). Thus the result should be viewed as a genuine fixed-dimensional \(\CCV\) lower
bound for \(\COCO\).  Unlike prior constructions whose horizon dependence comes
from letting the dimension grow, our instance lies in a fixed \(d\)-dimensional
Euclidean ball and applies to arbitrary online policies satisfying the stated
regret budget.  It shows that the nested geometry of adversarial constraints
alone can force nontrivial cumulative constraint violation.
\end{remark}

\subsection{Lower bound for Convex Losses} \label{cvx-loss-lb}
In this Section, we establish a minimax $\CCV$ lower bound for online algorithms with a \emph{weakly adaptive} regret bound. From the proof of Lemma \ref{lem:nested-ogd-regret}, it can be verified that the the \textsf{NP-OGD} algorithm is weakly adaptive. For a regret budget \(R(T)\), let \(V_{\rm wa}(R(T),T)\) denote the minimax
\(\CCV\) value over online algorithms satisfying the weak interval-regret
guarantee with budget \(R(T)\).

\begin{theorem}[$\CCV$ lower bound]
\label{thm:cvx-lower-bound}
Consider the above setup. Fix any online $\COCO$algorithm achieving a weakly adaptive sublinear regret of \(R(T)\). This means that for any interval $I \subseteq [T],$ and for any feasible benchmark $x^*$, the online algorithm achieves 
\begin{eqnarray*}
	 \Reg_I
        :=
        \sum_{t\in I}f_t(x_t)
        -
        \sum_{t \in I}f_t(x^*) \leq R(T).
\end{eqnarray*}
 Then there exists a $\COCO$problem instance on the $d$-dimensional unit ball with $1$-Lipschitz linear losses and $1$-Lipschitz affine constraints for which, the minimax $\CCV$ is lower bounded as: 
\begin{eqnarray*}
	V_{\rm wa}(R(T),T ) \geq m \varepsilon/2 = \Omega\bigg(\big(\frac{T}{R(T)}\big)^{1- \frac{4}{d+3}}\bigg)
\end{eqnarray*}
In particular, if the online algorithm achieves the minimax-optimal regret, \emph{i.e.,} $R(T) = O(\sqrt{T})$ as in \textsf{NP-OGD}, then we have 
\begin{eqnarray*}
	V_{\rm wa}(R(T),T) \ge \Omega(T^{\frac{1}{2}- \frac{2}{d+3}}).
\end{eqnarray*}

%
\end{theorem}
\begin{proof}

Our construction for convex losses is similar to the strongly convex case given above where the cost function does not involve the quadratic term, \emph{i.e.,}  the loss in phase $j = (k, u_i)$ is now taken to be
\begin{eqnarray} \label{loss-phase-j-cvx}
        f_t^{(j)}(x)
        :=
        -
        \langle u_i,x\rangle.
\end{eqnarray}
In the following, we outline the basic differences between two constructions. 
Take the $j$\textsuperscript{th} phase length to be at most 
\begin{eqnarray}
\label{phase-len-j}
        L_j
        =
        \left\lceil
        \frac{2}{\varepsilon}\big(R(T)+1\big)
        \right\rceil .
\end{eqnarray}
The trigger set $\mathcal{T}_j$ is defined exactly as in Eqn.\ \eqref{t-cap}.
 Similar to claim \ref{claim-trigger}, we now claim that a no-regret learner must play an action from the trigger set $\mathcal{T}_j.$
The incremental regret of the learner at any round $t$ is given by 
 \begin{eqnarray*}
	 f_t^{(j)}(x_t)-f_t^{(j)}(z_j) =  - \langle u_i, x_t \rangle  + \langle u_i, z_j\rangle 
	 =  -  s + h_j \geq \varepsilon/2.
 \end{eqnarray*}
 Thus the regret incurred on the $j$\textsuperscript{th} phase with respect to the feasible action $z_j$ is at least $ \frac{\varepsilon}{2}L_j \ge R(T)+1 > R(T).$ The above contradicts the weakly adaptive regret guarantee of the above algorithm over the the $j$\textsuperscript{th} phase. Thus, as before, the constraint violation at the final round of phase $j$ is given by 
\[g(x_{\tau_j}) = \langle u_i, x_{\tau_j} \rangle - (1-k\varepsilon) \stackrel{(a)}{\geq} 1 - (k-\frac{1}{2})\varepsilon - (1-k \varepsilon) \ge \varepsilon/2, \]

where inequality (a) follows from the fact that the action $x_{\tau_j}$ belongs to the trigger set $\mathcal{T}_j$ defined in Eqn.\ \eqref{t-cap}.

Now we define the number of phases $m = \left\lfloor
\min\left\{
c'_d\varepsilon^{-\frac{d+1}{2}},
\frac{\varepsilon T}{4(R(T)+1)}
\right\}
\right\rfloor,$ where $c_d'$ is some fixed dimension-dependent constant.
From the expression \eqref{phase-len-j}, it follows that all $m$ phases fit into a horizon of length $T.$ Finally, choosing the parameter $\varepsilon = (\frac{R(T)+1}{T})^{\frac{2}{d+3}},$ the total $\CCV$ over the entire horizon of length $T$ is given by 
\begin{eqnarray*}
	V_{\rm wa}(R(T),T ) \geq m \varepsilon/2 = \Omega\bigg(\big(\frac{T}{R(T)}\big)^{1- \frac{4}{d+3}}\bigg)
\end{eqnarray*}
In particular, if the online algorithm achieves the minimax-optimal regret, \emph{i.e.,} $R(T) = O(\sqrt{T})$, then we have 
\begin{eqnarray*}
	V_{\rm wa}(R(T),T ) \ge \Omega(T^{\frac{1}{2}- \frac{2}{d+3}}).
\end{eqnarray*}
\end{proof}
The above converse result, implying a fundamental trade-off between $\Regret$ and $\CCV$ should be compared with the result established in \cite{sinha2026beyond} for non-negative smooth and convex function, who established a similar achievability result for a gradient-based algorithm. Independent and concurrent work \cite{balasundaram2026lowerboundcumulativeconstrained} achieved a sharper lower bound of $\Omega(T^{\frac{d-1}{2d}})$ for the specific case of \textsf{NP-OGD} algorithm.

\section {The $\CExperts$Problem}
\label{experts-sec}

In this section, we consider a special case of the $\COCO$ problem, called the \textsf{Constrained Expert} problem
\cite[Section~3]{sinha2026beyond}.  There are
\(N\) experts corresponding to $N$ distinct actions.  On each round \(t\), the learner first chooses a probability distribution $p_t$ over the set of experts and then observes
a loss vector \(f_t\in[-1,1]^N\) and a constraint-violation vector
\(g_t\in[-1,1]^N\).  
As a result, the learner incurs a cost of $\langle f_t, p_t\rangle$ and a constraint violation of $\langle g_t, p_t \rangle_+.$
Under the \textsf{Common Feasibility} Assumption (Assumption \ref{common-feasibility}),  there exists at least one feasible expert
\(i^\star\in[N]\) such that
\[
        g_t(i^\star) \leq 0,
        \qquad
        \forall t\ge1.
\]
In this section, the goal is to obtain the minimax-optimal regret with respect to the \emph{best feasible expert} in hindsight while minimizing the $\CCV.$ To exploit the essentially discrete nature of the problem, unlike $\textsf{COCO},$ here we do not consider a general convex combination of experts as benchmarks (for which, we can just use the general results from the previous section). Instead, we work with a (weaker) discrete set of feasible experts as our benchmark.    


\subsection{Algorithm: Active Hedge with Elimination} 
For the $\CExperts$problem, we employ a natural and intuitive idea, in line with the general strategy outlined in Eqn.\ \eqref{gen-omd}. At the beginning of each round $t\ge 1,$ the algorithm keeps track of the set of currently \emph{Active Experts} $S_{t-1},$ defined in Eqn.\ \eqref{feas-set}, which satisfies all constraints presented up to round $t-1$. Then we use the exponential weighting mechanism, with the support of the distribution restricted to the active set $S_{t-1},$ for choosing the current action $p_t.$ Algorithm \ref{alg:active-hedge-full-info} gives the formal description of the policy. 

\begin{algorithm}[]
\caption{Active Hedge with Elimination for Constrained Experts}
\label{alg:active-hedge-full-info}
\begin{algorithmic}[1]
\REQUIRE Number of experts \(N\), learning rate \(\eta>0\).
\STATE Initialize \(S_0\gets[N]\) and \(w_1(i)\gets1\) for all \(i\in[N]\).
\FOR{\(t=1,\ldots,T\)}
    \STATE Play the distribution
    \[
        p_t(i)
        = \begin{cases}
        	        \frac{w_t(i)}{\sum_{j\in S_{t-1}}w_t(j)}, ~~ \textrm{if}~~ i \in S_{t-1}\\
0, ~~ \textrm{otherwise}
        \end{cases}
    \]
    \STATE Observe the cost and constraint vectors \(f_t,g_t\) and incur cost $\langle f_t, p_t\rangle$ and constraint violation $\langle g_t, p_t\rangle_+$.
    \STATE Update the set of active experts by removing any violating experts from the active set:
    \[
        S_{t}:=\{i\in S_{t-1}:g_t(i) \leq 0\}.
    \]
    \STATE Update the weights of the experts 
    \begin{eqnarray} \label{wt-ev}
    	        w_{t+1}(i)= \begin{cases} w_t(i)\exp(-\eta f_t(i)), ~~ \textrm{if} ~~ i \in S_{t} \\
 0, ~~ \textrm{otherwise}	
 \end{cases}
    \end{eqnarray}    
\ENDFOR
\end{algorithmic}
\end{algorithm}
Compared to Algorithm \ref{alg:nested-projected-ogd}, which involves a Euclidean projection step at each round, Algorithm \ref{alg:active-hedge-full-info} is particularly efficient as the currently active set of experts $S_t$ can be obtained by merely noting the sign of the components of the constraint vectors. The following theorem gives performance bounds for Algorithm \ref{alg:active-hedge-full-info}. The proof for the regret bound is very similar to the standard Hedge proof \citep{cesa2006prediction} after incorporating the observation that removing violating and hence, irrelevant experts, does not affect the final regret bound.

\begin{theorem}
\label{thm:active-hedge-distribution-version}
Consider the $\CExperts$problem with \(N\) experts. 
Let the set of feasible experts $\mathcal{X}^*$ be given by 
\[
       \mathcal{X}^* = S_T= \{i^\star \in [N]:  g_t(i^\star) \leq 0,~~
        \forall t \}.
\]
In this setting, Algorithm~\ref{alg:active-hedge-full-info} yields the following regret bound against all feasible experts:
\[
       \Regret_T(i^*) \equiv  \sum_{t=1}^T
        \langle f_t,p_t\rangle
        -
        \sum_{t=1}^T f_t(i^\star)
        \le
        \frac{\ln N}{\eta}
        +
        \frac{\eta T}{2},\quad \forall i^* \in \mathcal{X}^*,
\]
while simultaneously achieving a constant $\CCV$ independent of $T:$
\[
        \CCV_T \equiv \sum_{t=1}^T\langle g_t,p_t\rangle_+
        \le
        N.
\]
In particular, upon choosing the parameter
        $\eta=\sqrt{\frac{2\ln N}{T}},$
Algorithm \ref{alg:active-hedge-full-info} yields the following bounds:
\[
        \Regret_T(i^\star)=O(\sqrt{T\ln N})~~\forall i^* \in \mathcal{X}^*,
        \qquad ~ \textrm{and}~~~
        \CCV_T\le N.
\]
\end{theorem}
\paragraph{Remarks:} 
For sufficiently large horizons $(T \geq \tilde{\Omega}(N^2))$, Theorem~\ref{thm:active-hedge-distribution-version} improves the state-of-the-art trade-off by retaining the minimax-optimal regret bound of $O(\sqrt{T\ln N})$ while reducing the cumulative constraint violation from $O(\sqrt{T}\ln N\ln T)$, achieved by \cite[Theorem~2]{sinha2026beyond}, to $O(N)$.
\begin{proof}

By definition, any feasible expert \(i^\star \in \mathcal{X}^*\) is never removed from the active sets. 
We now define the potential function $W_t$ for the current active set $S_{t-1}:$
\[
        W_t:=\sum_{i\in S_{t-1}}w_t(i),
\]
where the weight vector $\bm{w}(t)$ evolves as in Eqn.\ \eqref{wt-ev}, with $\bm{w}(1)$ being the all one vector.  
Since $S_t \subseteq S_{t-1}$ and removing experts can only decrease the potential, we have 
\[
\begin{aligned}
        W_{t+1}
        &= \sum_{i \in S_t} w_{t+1} = 
        \sum_{i\in S_{t}}w_t(i)\exp(-\eta f_t(i))
        \le
        \sum_{i\in S_{t-1}}w_t(i)\exp(-\eta f_t(i))
        =
        W_t\sum_{i\in S_{t-1}}p_t(i)\exp(-\eta f_t(i)).
\end{aligned}
\]
Since \(f_t(i)\in[-1,1]\), by Hoeffding's lemma \citep[Lemma 2.2]{conc_ineq}, we have 
\[
        \sum_{i\in S_{t-1}}p_t(i)\exp(-\eta f_t(i))
        \le
        \exp\left(
            -\eta\langle f_t,p_t\rangle+\frac{\eta^2}{2}
        \right).
\]
Therefore,
\[
        \ln W_{t+1}
        \le
        \ln W_t
        -
        \eta\langle f_t,p_t\rangle
        +
        \frac{\eta^2}{2}.
\]
Summing the above inequalities from \(t=1\) to \(T\), we obtain
\[
        \ln W_{T+1}
        \le
        \ln N
        -
        \eta\sum_{t=1}^T\langle f_t,p_t\rangle
        +
        \frac{\eta^2T}{2},
\]
where we have used the fact that $W_1=N.$ On the other hand, for any \(i^\star\in S_{T+1} = \mathcal{X}^*\), we have
\[
        W_{T+1}
        \ge
        w_{T+1}(i^\star)
        =
        \exp\left(
            -\eta\sum_{t=1}^T f_t(i^\star)
        \right) 
       \implies  \ln W_{T+1}
        \ge
        -\eta\sum_{t=1}^T f_t(i^\star).
\]
Combining the above two bounds on \(\ln W_{T+1}\), we conclude that for any feasible expert $i^* \in \mathcal{X}^*,$ we have
\[
        -\eta\sum_{t=1}^T f_t(i^\star)
        \le
        \ln N
        -
        \eta\sum_{t=1}^T\langle f_t,p_t\rangle
        +
        \frac{\eta^2T}{2}.
\]
Rearranging the above yields the regret bound: \[
    \Regret_T(i^*) =    \sum_{t=1}^T
        \langle f_t,p_t\rangle
        -
        \sum_{t=1}^T f_t(i^\star)
        \le
        \frac{\ln N}{\eta}
        +
        \frac{\eta T}{2}.
\]
To establish the $\CCV$ bound, let
\[
        B_t:=\{i\in S_{t-1}:g_t(i)>0\}
\]
be the set of active experts eliminated at round \(t\).  Clearly, the sets \(B_t\) are
disjoint over time.  Since \(g_t(i)\in[-1,1]\), we have
\[
        \langle g_t,p_t\rangle_+
        \leq 
        \sum_{i\in B_t}p_t(i)g_t(i)
        \le
        \sum_{i\in B_t}p_t(i)
        \le
        |B_t|.
\]

Therefore,
\[
        \sum_{t=1}^T\langle g_t,p_t\rangle_+
        \le
        \sum_{t=1}^T |B_t|
        \le
        N.
\]
Finally, choosing the parameter \(\eta=\sqrt{2\ln N/T}\) yields the desired result.
\end{proof}
\subsection{Lower bounds for the $\CExperts$Problem} \label{sec:cexperts-lb}
It is well-known that the minimax optimal regret bound for the standard $\textsf{Experts}$ problem is $\Theta(\sqrt{T \log N})$ \citep{cesa2006prediction}. Thus the regret bound given by Theorem \ref{thm:active-hedge-distribution-version} is already minimax-optimal for the $\CExperts$problem. Hence, we only need to establish an $\Omega(N)$ lower bound for $\CCV$ to prove the optimality of Algorithm \ref{alg:active-hedge-full-info}. Similar to the $\COCO$setting, for a given sublinear regret budget \(R(T)\), we define the minimax $\CCV$ value as 
\begin{align}
\label{reg-bud-lb-expert}
        V_N(R(T),T)
        :=
         \inf_{\mathcal A:\sup_I \Reg_T(\mathcal A, I)\le R(T)} \sup_I
         \CCV_T(\mathcal A,I),
\end{align}
where $I$ is an arbitrary problem instance for the $\CExperts$problem and the outer infimum is taken over all online policies $\mathcal{A}$ satisfying the regret budget $R(T)$ for all valid instances. Our goal in this section is to show that $V_N(R(T),T) \geq \Omega(N).$ 
Note that this lower bound is non-trivial as Section \ref{sec:pure-feasibility-experts} in the Appendix shows that in the special case when all cost vectors are identically equal to zero, a very simple algorithm achieves $O(\ln N)$ $\CCV$. This makes it clear that our lower bound strategy cannot look at the constraints in isolation but must integrate the regret budget $R(T)$ into it deeply. 
In the following, we construct an adversarial instance for the $\CExperts$problem that leads to the above lower bound.

\paragraph{Construction of the Adversarial Instance:}
Our $\CCV$ lower bound proof is based on dividing the entire time-interval into multiple phases where each phase corresponds to an expert. Formally, fix any algorithm \(\mathcal A\) for the $\CExperts$problem satisfying the following regret bound:
\[
        \sup_I \Reg_T(\mathcal A,I)\le R(T),
\]
where $R(T)$ is any sublinear function and the supremum is taken over all valid problem instances. Phase $j$ begins at round $P_j+1$ (we set  
        $P_1:=0$) and continues for at most $L_j$ rounds where $L_j\ge4(R(T)+2P_j+1)$. In phase \(j \in [N]\), the adversary makes expert
\(j\) uniquely attractive by making all loss components corresponding to expert $j$ equal to zero and all other components equal to one. Furthermore, at all rounds of phase $j,$ except the last round, the adversary outputs trivial constraint function $g_t=0.$ In other words, for any round $t$ in the $j$\textsuperscript{th} phase, we set:

\[
        f_t(j)=0,
        \qquad
        f_t(i)=1 \quad \text{for } i\neq j,
\]
and all constraints equal to zero:
\[
        g_t(i)=0,
        \qquad
        i \in [N],
\]

apart from the last round of phase $j,$ when we set $g_t(j)=1.$

Intuitively, if the learner does not place sufficient probability mass on expert \(j\) throughout the $j$\textsuperscript{th} phase, then the
adversary could simply keep expert \(j\) feasible, and the learner would incur
a large regret, violating its promised sublinear regret bound of $R(T)$.  Therefore a low-regret learner must place a large
mass on expert \(j\) at some point during the $j$\textsuperscript{th} phase. At the same round, the adversary reveals a linear constraint function which makes expert $j$ infeasible and phase $j$ ends incurring a constraint violation for the learner. We continue this construction for all $N-1$ phases. The following lemma formalizes this intuition.
\begin{lemma}[The learner must play Expert $j$ on the $j$\textsuperscript{th} phase]
\label{lem:experts-phase-forcing1}
Consider the problem instance constructed above. Then there must exist a round $\tau\in\{P_j+1,\ldots,P_j+L_j\}$ in phase $j$
such that
\[
        p_\tau(j)\ge \frac34,
\]
where $p_t$ is the action of the learner at round $t$ as defined above. 
\end{lemma}

\begin{proof}
We argue by contradiction.  Suppose
\begin{eqnarray} \label{contradiction}
        p_t(j)<\frac34
        \qquad
        \forall t \in \{P_j+1,\ldots,P_j+L_j\}.
\end{eqnarray}
We now extend the above instance by setting all loss and constraint vectors to be identically equal to zero after round $P_j+L_j.$ Thus Expert $j$ remains feasible throughout the end of the horizon. We can now lower bound the regret of the learner for this input instance against the $j$\textsuperscript{th} Expert.

The first \(P_j\) rounds before the phase $j$ can trivially contribute at least \(-2P_j\) regret relative to expert
\(j\), because all losses lie in \([-1,1]\).  At any round $t$ in the $j$\textsuperscript{th} phase, the
learner's loss on round \(t\) is given by 
\[
        \langle p_t,f_t\rangle
        =
        1-p_t(j),
\]
as all experts except the \(j\)\textsuperscript{th} incur loss one during this period. Hence, we have 
\[
\begin{aligned}
      \Reg_T(j)
        &\ge
        -2P_j+
        \sum_{t=P_j+1}^{P_j+L_j}(1-p_t(j)) + \underbrace{0}_{\textrm{regret beyond phase $j$}} \\
        &\stackrel{\textrm{Eq.}~\eqref{contradiction}}{>}
        -2P_j+\frac{L_j}{4}.
\end{aligned}
\]

Since we assumed the length $L_j$ of the $j$\textsuperscript{th} phase to be at least 
$4(R(T)+2P_j+1)$, we conclude 
\[
       \Reg_T(j) \geq -2P_j+\frac{L_j}{4}
        \ge
        R(T)+1
        >
        R(T).
\]

This contradicts the assumption that \(\mathcal A\) has a regret of at most
\(R(T)\) for every input instance.  Therefore it follows that some round in phase $j$ must satisfy
\[
        p_\tau(j)\ge\frac34.
\]
\end{proof}

\begin{theorem}[Minimax $\CCV$ lower bound for $\CExperts
$]
\label{thm:experts-omega-N-lower-bound}
There exists a numerical constant \(c>0\) such that, for every \(N\ge2\), every
horizon \(T\), and every regret budget \(R(T)\ge0\), the minimax $\CCV$, defined in Eqn.\ \eqref{reg-bud-lb-expert}, is lower bounded as 
\[
        V_N(R,T)
        \ge
        c\,
        \min\left\{
            N-1,\,
            \left\lfloor
            \log_{10}\left(\frac{T}{R(T)+1}\right)
            \right\rfloor_+
        \right\},
\]
where
\[
        \lfloor a\rfloor_+:=\max\{0,\lfloor a\rfloor\}.
\]
In particular, if
\[
        T\ge (R(T)+1)10^{N-1},
\]
which holds for large enough $T$ as long as $R(T)$ is sublinear,
then
\[
        V_N(R,T)=\Omega(N).
\]
\end{theorem}

\begin{proof}
Fix an algorithm \(\mathcal A\) satisfying
\[
        \sup_I\Reg_T(\mathcal A,I)\le R(T).
\]
Let
\begin{eqnarray}   \label{M-def}     
 M:=
        \min\left\{
            N-1,\,
            \left\lfloor
            \log_{10}\left(\frac{T}{R(T)+1}\right)
            \right\rfloor_+
        \right\}.
        \end{eqnarray}

If \(M=0\), there is nothing to prove, so we assume \(M\ge1\).
Consider the adversarial instance with $M \leq N$ phases as constructed above. Lemma \ref{lem:experts-phase-forcing1} ensures that there exists a round $\tau_j$ in phase $j$ when $p_{\tau_j}(j) \geq \frac{3}{4}.$ 
At this round (which happens to be the last round of phase $j$), the adversary reveals the following constraint vector
\[
        g_{\tau_j}(j)=1,
        \qquad
        g_{\tau_j}(i)=0\quad\text{for }i\neq j.
\]

Therefore, the violation contributed by phase \(j\) is at least
\[
        \langle p_{\tau_j}, g_{\tau_j}\rangle
        =
        p_{\tau_j}(j)
        \ge
        \frac34.
\]

Repeating this construction for \(j=1,\ldots,M\), we conclude that the above adversarial instance yields the following lower bound on the $\CCV:$
\[
        \CCV_T\ge \frac34 M.
\]
Finally, it remains to verify that all $N$ phases fit inside the horizon of length $T$. For this, define the quantity
\[
        Q_j:=P_j+R(T)+1.
\]
Since the length of the $j$\textsuperscript{th} phase is at most
\[
        L_j\le4(R(T)+2P_j+1)+1\leq 8Q_j+1,
\]
we have
\begin{eqnarray*}
        Q_{j+1}
        =
        P_{j+1}+R(T)+1
        \leq 
        P_j+L_j+R(T)+1
        = Q_j+L_j \le 
        Q_j+8Q_j+1
        \le
        10Q_j,
\end{eqnarray*}
where, in the last inequality, we have used the fact that \(Q_j\ge1\).  Therefore
\[
        Q_{M+1}\le10^MQ_1=10^M(R(T)+1).
\]
By the definition of \(M\) given in \eqref{M-def}, we have
\[
        10^M(R(T)+1)\le T.
\]
Thus
\[
        P_{M+1}\le Q_{M+1}\le T,
\]
so all $M$ phases fit inside the horizon. 
%
This concludes the proof of the $\CCV$ lower bound.
\end{proof}

\section{Conclusion and an Open Problem}
\label{sec:conclusion}
In this paper, we considered Constrained Online Convex Optimization (COCO) under the common feasibility assumption, demonstrating that the nested geometry of the successive feasible sets can be exploited for obtaining improved violation bounds. We focus on the fixed-dimensional setting and examine the large-$T$ regime. Using a simple projected gradient-based algorithm, we obtain $\mathcal{O}(\log T)$ bounds for both regret and $\CCV$ for strongly convex losses - an exponential improvement of $\CCV$ over the state-of-the-art result while retaining the minimax optimal regret bound. Furthermore, we achieve the $\mathcal{O}(\sqrt{T})$ $\Regret$ and $O(\sqrt{T})$ $\CCV$ bounds for general convex losses, improving the $\CCV$ bound by a factor of $O(\log T)$. These results are established by exploiting the approximate self-contraction property of the descent curve traced out by the projected gradient descent dynamics. 
We also study the special case of the $\CExperts$ problem and complement our results with matching lower bounds. 
We leave the question of obtaining minimax-optimal rates jointly for an arbitrary dimension $d$ and horizon length $T$ as an open problem. 

\section{Acknowledgement}

This work was supported by the
Department of Atomic Energy, Government of India, under project no. RTI4014 and by a Google India faculty research award.

\clearpage
\bibliographystyle{plain}
\bibliography{refs}
\clearpage

\appendix

\section{Appendix}
\subsection{Proof of Lemma \ref{lem:nested-ogd-regret}} 
\label{reg-bd-proof}
\begin{proof}
The argument closely follows the standard regret analysis of the OGD algorithm with one minor but important
modification: the projection set changes with each round. The standard proof goes through precisely because of the nested structure of the projection step. 

Fix \(x^\star\in S_T\).  The update is
\[
        x_{t+1}
        =
        \Pi_{S_t}(x_t-\eta_t h_t),
        \qquad
        h_t\in\partial f_t(x_t).
\]
Since \(x^\star\in S_t\), non-expansiveness of Euclidean projection gives
\[
        \|x_{t+1}-x^\star\|_2^2
        \le
        \|x_t-\eta_t h_t-x^\star\|_2^2 .
\]
Expanding the right-hand side,
\[
        \|x_{t+1}-x^\star\|_2^2
        \le
        \|x_t-x^\star\|_2^2
        -
        2\eta_t\langle h_t,x_t-x^\star\rangle
        +
        \eta_t^2\|h_t\|_2^2 .
\]
Therefore,
\[
        \langle h_t,x_t-x^\star\rangle
        \le
        \frac{
        \|x_t-x^\star\|^2-\|x_{t+1}-x^\star\|^2
        }{2\eta_t}
        +
        \frac{\eta_t}{2}\|h_t\|^2 .
\]

Since \(f_t\) is \(H_t\)-strongly convex, with \(H_t=0\) allowed, we have
\[
        f_t(x_t)-f_t(x^\star)
        \le
        \langle h_t,x_t-x^\star\rangle
        -
        \frac{H_t}{2}\|x_t-x^\star\|^2 .
\]
Combining the last two inequalities and using \(\|h_t\|_2\le G\), we obtain
\[
\begin{aligned}
        f_t(x_t)-f_t(x^\star)
        &\le
        \frac12\left(
            \frac{1}{\eta_t}
            -
            H_t
        \right)
        \|x_t-x^\star\|^2
        -
        \frac{1}{2\eta_t}
        \|x_{t+1}-x^\star\|^2
        +
        \frac{\eta_tG^2}{2}.
\end{aligned}
\]

Summing over \(1\leq t \leq T\), the RHS yields
\[
\begin{aligned}
&\sum_{t=1}^T
\left[
        \frac12\left(\frac1\eta_t-H_t\right)\|x_t-x^\star\|^2
        -
        \frac{1}{2\eta_t}\|x_{t+1}-x^\star\|^2
\right]   + \frac{G^2}{2}
        \sum_{t=1}^T\eta_t \\
&\qquad =
        \frac12(\frac{1}{\eta_1}-H_1)\|x_1-x^\star\|^2
        +
        \frac12\sum_{t=2}^T
        (\frac{1}{\eta_t}-\frac{1}{\eta_{t-1}}-H_t)\|x_t-x^\star\|^2
        -
        \frac{1}{2\eta_T}\|x_{T+1}-x^\star\|^2 .
\end{aligned}
\]
Dropping the final nonpositive term and using
\[
        \|x_t-x^\star\|\le D
        \qquad
        \forall t,
\]
we obtain
\[
\begin{aligned}
        \sum_{t=1}^T
        \bigl(f_t(x_t)-f_t(x^\star)\bigr)
        &\le
        \frac{D^2}{2}(\frac{1}{\eta_1}-H_1)_+
        +
        \frac{D^2}{2}\sum_{t=2}^T
        (\frac{1}{\eta_t}-\frac{1}{\eta_{t-1}}-H_t)_+  
        +
        \frac{G^2}{2}\sum_{t=1}^T\eta_t .
\end{aligned}
\]
Equivalently,
\[
        \sum_{t=1}^T
        \bigl(f_t(x_t)-f_t(x^\star)\bigr)
        \le
        \frac{D^2}{2}
        \sum_{t=1}^T
        (\frac{1}{\eta_t}-\frac{1}{\eta_{t-1}}-H_t)_+
        +
        \frac{G^2}{2}
        \sum_{t=1}^T\eta_t,
\]
where we have used the convention that $\eta_0^{-1}\equiv 0.$
Taking \(x^\star\in\arg \min_{x\in S_T}\sum_{t=1}^T f_t(x)\) gives the regret
bound.
\end{proof}

\subsection{Pure Feasibility for the $\CExperts$Problem}
\label{sec:pure-feasibility-experts}
In this section, we study a special case of the $\CExperts$problem where all loss vectors are set identically equal to zero (and hence, the regret for any policy is trivially equal to zero). 
In this case, we show that Algorithm \ref{alg:uniform-active-elimination}, described below, achieves \(O(\log N)\) $\CCV$. 

\begin{algorithm}[H]
\caption{Uniform Active Elimination}
\label{alg:uniform-active-elimination}
\begin{algorithmic}[1]
\STATE Initialize the active set \(S_1\gets[N]\).
\FOR{\(t=1,\ldots,T\)}
    \STATE Play the uniform distribution on the active experts:
    \[
        p_t(i)
        =
        \begin{cases}
        1/|S_t|, & i\in S_t,\\
        0, & i\notin S_t.
        \end{cases}
    \]
    \STATE Observe \(g_t\).
    \STATE Delete all active experts that violate:
    \[
        S_{t+1}
        :=
        \{i\in S_t:g_t(i)\le0\}.
    \]
\ENDFOR
\end{algorithmic}
\end{algorithm}

\begin{theorem}[Logarithmic CCV for pure feasibility $\CExperts$problem]
\label{thm:pure-feasibility-experts-upper}
Under common feasibility, Algorithm~\ref{alg:uniform-active-elimination}
satisfies
\[
        \CCV_T
        :=
        \sum_{t=1}^T\langle p_t, g_t\rangle_+
        \le
        H_N
        \le
        1+\log N,
\]
where
\[
        H_N:=\sum_{m=1}^N\frac1m
\]
is the \(N\)-th harmonic number.
\end{theorem}

\begin{proof}
Since there exists an expert \(i^\star\) with \(g_t(i^\star)\le0\) for every
\(t\), the active set is never empty.

Let $m_t:=|S_t|$ be the number of active experts at round $t$ and 
let $B_t:=\{i\in S_t:g_t(i)>0\}$ 
be the set of active experts eliminated on round \(t\), and define $b_t:=|B_t|.$ 
Since \(p_t\) is uniform on \(S_t\), and \(g_t(i) \le1\),  the constraint violation at round $t$ can be expressed as:
\begin{eqnarray} \label{eq-ccv-bd}
        \langle p_t,g_t\rangle_+
        \le 
        \frac1{m_t}\sum_{i\in B_t}g_t(i)
        \le
        \frac{b_t}{m_t}.
\end{eqnarray}
After the deletion step, the number of active experts at round $t$ reduces to 
\[
        m_{t+1}=m_t-b_t.
\]
If \(b_t=0\), the round contributes no violation.  If \(b_t>0\), then
\[
        \frac{b_t}{m_t}
        \le
        \sum_{s=m_t-b_t+1}^{m_t}\frac1s
        =
        \sum_{s=m_{t+1}+1}^{m_t}\frac1s,
\]
because each of the \(b_t\) summands on the right is at least \(1/m_t\) and by definition, $m_t >0, \forall t.$
Therefore, from Eqn.\ \eqref{eq-ccv-bd}, we conclude that
\[
\begin{aligned}
        \CCV_T = \sum_{t=1}^T \langle p_t, g_t \rangle 
        &\le
        \sum_{t=1}^T
        \sum_{s=m_{t+1}+1}^{m_t}\frac1s \le
        \sum_{s=1}^N\frac1s
        =
        H_N
        \le
        1+\ln N.
\end{aligned}
\]
\end{proof}

\end{document}